%% file: main.tex
\documentclass[sigconf]{acmart}
\AtBeginDocument{%
  \providecommand\BibTeX{{%
    \normalfont B\kern-0.5em{\scshape i\kern-0.25em b}\kern-0.8em\TeX}}}

\copyrightyear{2021}
\acmYear{2021}
\setcopyright{acmcopyright}\acmConference[PODC '21]{Proceedings of the 2021 ACM Symposium on Principles of Distributed Computing}{July 26--30, 2021}{Virtual Event, Italy}
\acmBooktitle{Proceedings of the 2021 ACM Symposium on Principles of Distributed Computing (PODC '21), July 26--30, 2021, Virtual Event, Italy}
\acmPrice{15.00}
\acmDOI{10.1145/3465084.3467919}
\acmISBN{978-1-4503-8548-0/21/07}

\usepackage{hyperref}
\usepackage{microtype}
\usepackage{graphicx}
\usepackage{subfigure}
\usepackage{booktabs} 
\usepackage{amsmath,amsthm,amsfonts}
\usepackage[ruled,vlined,linesnumbered]{algorithm2e}
\usepackage{graphicx}
\usepackage{textcomp}
\usepackage{xcolor}
\usepackage{float}
\usepackage{comment,bm}
\usepackage{caption}
\usepackage{titlesec}
\usepackage{verbatim}
\usepackage[english]{babel}
\usepackage[autostyle, english = american]{csquotes}
\MakeOuterQuote{"}
\usepackage[shortlabels]{enumitem}

\usepackage{hyperref}
\hypersetup{
    colorlinks=true,
    linkcolor=blue,
    filecolor=magenta,      
    urlcolor=cyan,
    pdftitle={Overleaf Example},
    pdfpagemode=FullScreen,
    }

\theoremstyle{definition}

\newcounter{relctr} 
\everydisplay\expandafter{\the\everydisplay\setcounter{relctr}{0}} 

\newtheorem{assumption}{\bfseries Assumption}
\newtheorem{theorem}{\bfseries Theorem}
\newtheorem*{theorem*}{\bfseries Theorem}
\newtheorem{proposition}{\bfseries Proposition}
\newtheorem*{proposition*}{\bfseries Proposition}

\newtheorem{remark}{\bfseries Remark}

\newtheorem*{sketchofproof}{\bfseries Sketch of proof}

\providecommand{\iprod}[2]{\ensuremath{\left\langle #1,\,#2  \right\rangle}}

\def\R{\mathbb{R}}

\AtBeginDocument{} 

\input{common}
\DeclareMathOperator{\E}{\mathbb{E}}

\begin{document}
\fancyhead{}

\title{\papertitle{}}

\author{Rachid Guerraoui}
\email{rachid.guerraoui@epfl.ch}
\affiliation{
  \institution{Ecole Polytechnique Fédérale de Lausanne (EPFL)}
  \city{Lausanne}
  \country{Switzerland}
}

\author{Nirupam Gupta}
\email{nirupam.gupta@epfl.ch}
\affiliation{
  \institution{Ecole Polytechnique Fédérale de Lausanne (EPFL)}
  \city{Lausanne}
  \country{Switzerland}
}

\author{Rafaël Pinot}
\email{rafael.pinot@epfl.ch}
\affiliation{
  \institution{Ecole Polytechnique Fédérale de Lausanne (EPFL)}
  \city{Lausanne}
  \country{Switzerland}
}

\author{Sébastien Rouault}
\email{sebastien.rouault@epfl.ch}
\affiliation{
  \institution{Ecole Polytechnique Fédérale de Lausanne (EPFL)}
  \city{Lausanne}
  \country{Switzerland}
}

\author{John Stephan}
\email{john.stephan@epfl.ch}
\affiliation{%
  \institution{Ecole Polytechnique Fédérale de Lausanne (EPFL)}
  \city{Lausanne}
  \country{Switzerland}
}


\begin{abstract}
This paper addresses the problem of combining Byzantine resilience with privacy in machine learning (ML).
Specifically, we study if a distributed implementation of the renowned Stochastic Gradient Descent (SGD) learning algorithm is feasible with \textit{both} differential privacy (DP) and $(\alpha,f)$-Byzantine resilience. 
To the best of our knowledge, this is the first work to tackle this problem from a theoretical point of view. A key finding of our analyses is that the classical approaches to these two (seemingly) orthogonal issues are incompatible. 
More precisely, we show that a direct composition of these techniques makes the guarantees of the resulting SGD algorithm depend unfavourably upon the number of parameters of the ML model, making the training of large models practically infeasible.
We validate our theoretical results through numerical experiments on publicly-available datasets; showing that it is impractical to ensure DP and Byzantine resilience simultaneously.
\end{abstract}

\ccsdesc[300]{Security and privacy~Privacy-preserving protocols}
\ccsdesc[500]{Mathematics of computing~Continuous optimization}

\keywords{Machine learning, Differential privacy, Byzantine resilience, SGD}

\maketitle

\section{Introduction}
The massive amounts of data generated daily call for distributed machine learning (ML). Essentially, different nodes collaborate to train a joint model on a collective dataset.
Clearly, an aggregate model would be more accurate than individually-trained models on small subsets of data.
However, two reasons prevent the explicit sharing of personal data. Firstly, in many classification tasks, training data is sensitive and should remain private, e.g., financial and medical fields. Secondly, datasets can be quite large (e.g., Open Images \cite{openimages}, ImageNet \cite{imagenet}) and their sharing computationally expensive.\\

\noindent
The most popular scheme to train ML models in a distributed setting is Stochastic Gradient Descent (\textbf{SGD}) \cite{SGD}: it enables to train the aggregate model by simply exchanging gradients of the loss function (instead of the training data itself). Underlying SGD lies an iterative method to optimize the objective function $Q(w)$ by stochastically estimating the gradient $\nabla Q(w)$ and then computing a gradient descent step on $w$. There are several system models for distributed SGD training, such as the parameter server~\cite{paramServer} and ring all-reduce models~\cite{reduce}. The parameter server model is one of the most adopted distributed learning topologies (Fig.~\ref{fig:sysmodel}), where nodes send their gradients to a central trusted entity, namely the \textbf{parameter server}, responsible of updating the model parameters by aggregating the received gradients. The parameter server model is also the backbone of the popular setting in distributed learning today, \textit{Federated Learning}~\cite{federated}. Averaging the received gradients is typically used by the parameter server as aggregation method~\cite{averaging}, assuming the nodes correctly compute unbiased estimates of the gradient. However, releasing gradients in a distributed framework results in the emergence of two orthogonal threats: Byzantine gradients and data leakage.\\

\noindent
\paragraph{Byzantine Gradients} The learning can be critically influenced by Byzantine gradients (i.e., vectors that are not unbiased estimates of the true gradient) sent by the nodes during the training \cite{little, empire}. We call these gradients, as well as the nodes that send them, \textit{Byzantine}. We can distinguish two types of Byzantine gradients: \textbf{erroneous} gradients that correspond to arbitrary failures during the gradient computation (e.g., software bugs, loss of precision, mislabeling in local dataset, network asynchrony), and \textbf{malicious} gradients which are forged vectors sent by malicious participants in an attempt to poison the learning. In both cases, the injected Byzantine gradients can prevent the collective model from converging to a \textit{satisfying state}, i.e., a final accuracy that is \textit{comparable} to the accuracy resulting from a training with no Byzantine workers.

With Byzantine nodes hindering the training, simply averaging the received gradients prior to updating the global model parameters is not Byzantine-resilient. Consequently, several gradient aggregation rules (GAR), such as \krum{} \cite{krum}, \brute{} \cite{MDA}, or \median{} \cite{median} have been designed to tolerate a certain threshold of Byzantine nodes in the network.\\

\noindent
\paragraph{Data Privacy} ML models are known to leak information about their training data \cite{shokri, Mlleaks, DLG}. In the context of distributed SGD training, despite the existence of several network architectures, gradients are typically exchanged in the clear (unencrypted) between the different nodes, potentially causing data leaks. Zhu et al. recently developed an attack~\cite{DLG} confirming that gradients inherently contain information about the training samples that a \textbf{curious parameter server} can exploit to violate the privacy of data nodes. Several privacy-preserving implementations of the distributed SGD scheme have been developed. Among them, the dominant technique is to inject Gaussian or Laplace noise to the gradients to ensure differential privacy (DP)~\cite{DP_SGD_Fed1, DP_SGD_Fed2}. In addition to being the gold standard for data privacy in the ML community, DP enables us to theoretically study the privacy guarantees it provides in the presence of a \textbf{curious parameter server}, thanks to two privacy parameters $\epsilon$ and $\delta$. Reasonable privacy budgets $(\epsilon,\delta)$ typically lie in the range $(0,1)^2$ \cite{dwork2014algorithmic}.\\

\begin{figure}
\centering
\subfigure[Trusted Parameter Server With 8 (Honest) Workers]{\label{fig:sysmodel}\includegraphics[width=60mm]{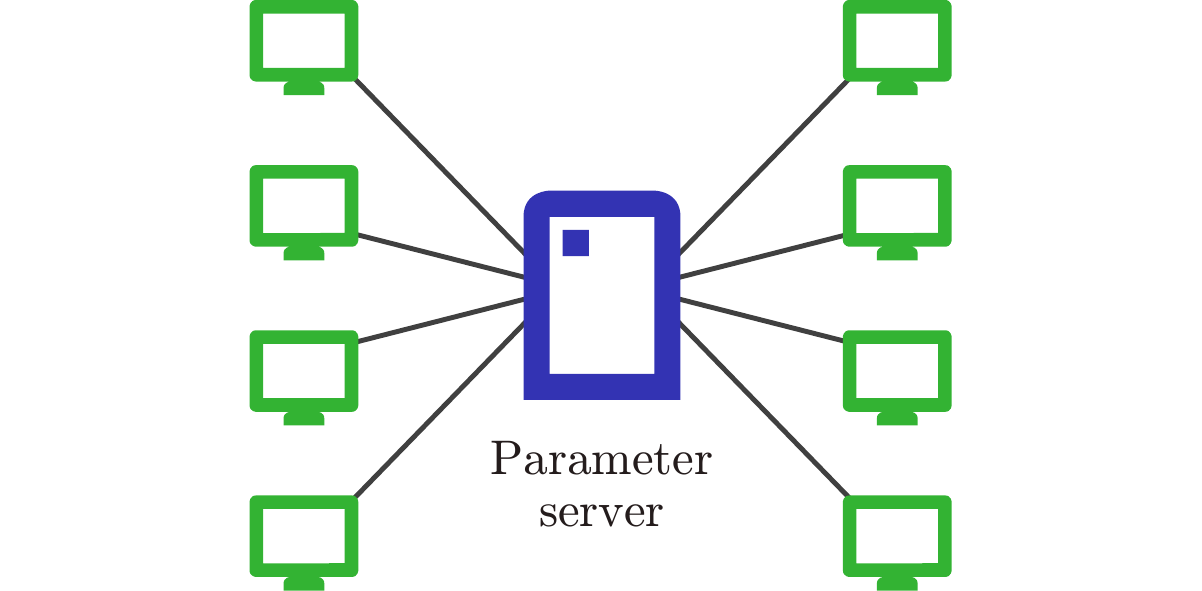}}
\subfigure[Honest-but-Curious Parameter Server With 8 Workers, Including 3 Byzantine Nodes]{\label{fig:sysmodel2}\includegraphics[width=60mm]{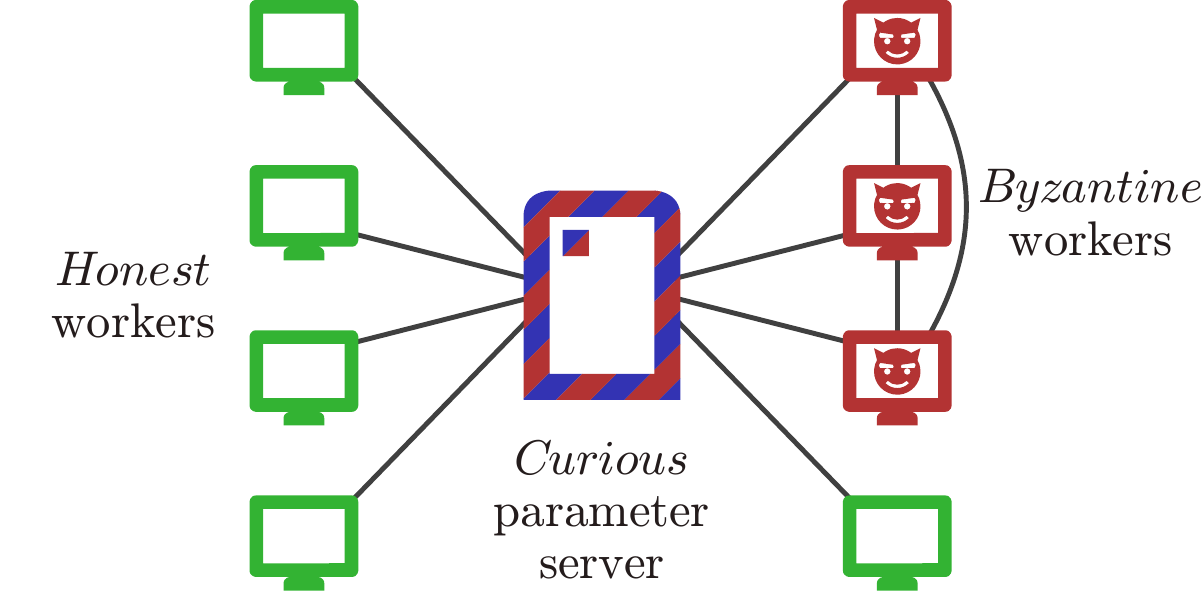}}
\caption{Parameter Server Model}
\end{figure}

\noindent
\paragraph{Problem Statement}
In a distributed SGD framework, the two aforementioned threats have been mainly tackled separately. 
We are the first to theoretically study the possibility of satisfying both $(\alpha,f)$-Byzantine resilience~\cite{krum} and DP in a decentralized SGD system.
In this paper, we study the possibility of combining $(\alpha,f)$-Byzantine resilience with DP via noise injection. To this end, we adopt the (now classical) parameter server model with a total of $n$ workers among which a maximum of $f \leq n$ workers can be Byzantine and may collude. Furthermore, we consider the parameter server to be \textbf{honest-but-curious}, meaning that it correctly computes aggregate gradients in each step of the training but it can also use the received gradients to violate the privacy of the honest workers. The resulting framework is illustrated in Fig.~\ref{fig:sysmodel2}. 
At first glance, combining $(\alpha,f)$-Byzantine resilience with DP by noise injection in distributed SGD seems like a simple extension of the Byzantine-resilient scenario with no privacy. In fact, it does not require any modification to the classical distributed SGD protocol or even the infrastructure of the parameter server model (no encryption protocol, no key distribution, no secret sharing, etc.). Furthermore, the noise is injected locally by every honest worker. Since there are existing solutions for $(\alpha,f)$-Byzantine resilient SGD, controlling the levels of noise added to the gradients should be sufficient to enforce Byzantine resilience while still achieving moderate but acceptable levels of privacy. However, our results show that $(\alpha,f)$-Byzantine resilience and DP do not add up.\\

\noindent
\paragraph{Contributions}
We show that combining the two orthogonal notions, DP and $(\alpha,f)$-Byzantine resilience, depends on the number of parameters (denoted by $d$) of the learning model. In short, we show that when $d$ is large, DP and Byzantine resilience do {\bf not} fit well together. We first study the general case when the cost function, associated with the learning problem, may be non-convex. In this particular case, we show that in order to guarantee $(\alpha, f)$-Byzantine resilience when DP noise is injected, either the batch size must grow linearly with $\sqrt{d}$, or the fraction of Byzantine workers in the system $\frac{f}{n}$ must decrease with $\sqrt{d}$. Given that $d$ is indeed one of the largest variables in contemporary learning problems~\cite{bottou2018optimization, lecun1998gradient, simard2003best}, and that certain state-of-the-art models often exceed 100 M parameters~\cite{DBLP:journals/corr/ZagoruykoK16}, this dependence on $d$ clearly highlights the impracticality of satisfying both DP and Byzantine resilience. Our results apply to the most well-known statistically-robust and $(\alpha, f)$-Byzantine resilient GARs such as \brute{}~\cite{MDA}, \krum{}~\cite{krum}, etc. For example, for \brute{}, we state the following result.

\begin{proposition*}[Informal]
    Let $(\epsilon,\delta) \in (0,1)^2$ be the privacy budget used to inject DP noise to the gradients, and $b$ be the batch size used. Then, we can only guarantee $(\alpha, f)$-Byzantine resilience if the fraction of Byzantine nodes is in $O\left(\frac{b}{\sqrt{d} + b}\right)$.
\end{proposition*}

\noindent   
 Next, for a more fine-grained analysis on the impact of DP noise on the Byzantine-resilient SGD scheme, we study a less general setting where the cost function is assumed \textbf{strongly-convex}. In this particular case, we show that for \textbf{any} (existing or non-existing) $(\alpha,f)$-Byzantine resilient GAR, the SGD algorithm with DP noise still suffers from the {\it curse of dimensionality}. Specifically, we show that the training error rate may grow linearly with $d$. On the other hand, the training error rate of the same algorithm without DP noise is independent of $d$.
 
\begin{theorem*}(Informal) The distributed SGD algorithm combining $(\alpha, f)$-Byzantine resilience and $(\epsilon,\delta)$-DP has a training error rate of $\Theta\left( \frac{d \log(1/\delta)}{ T b^2 \epsilon^2} \right)$, where $T$ denotes the number of steps and $b$ is the batch size.
\end{theorem*}

\noindent
Note that in this work, we employ the Gaussian mechanism \cite{DP} to ensure $(\epsilon, \delta)$-DP. Nonetheless, our key findings on the impracticality of simultaneously ensuring DP and Byzantine resilience remain unchanged when adapting our results to support other noise injection techniques such as the Laplacian mechanism \cite{DP}.

\section{Background}
\subsection{The Classical Distributed SGD Model}\label{DistSGD}
Let us first recall how the distributed SGD protocol works in the classical (simplified) scenario where all $n$ workers are assumed to be honest. Let $Q$ be the cost function to minimize and $\mathcal{D}$ the ground-truth data distribution. Let $t$ be the current training step and $\weight{t} \in \mathbb{R}^d$ be the model parameters at step $t$. Each worker \worker{}{i} locally samples a random training batch \batch{i}{t} from the data distribution $\mathcal{D}$ to compute an unbiased estimate \gradient{i}{t} of the gradient $\nabla Q(\weight{t})$. This means that, \gradient{1}{t}, ..., \gradient{n}{t} are i.i.d. random vectors such that $\E_{\batch{i}{t}}\left[\gradient{i}{t}\right] = \nabla Q(\weight{t})$. Then, all the workers send their gradients to the parameter server. The training is divided into sequential synchronous steps, hence the parameter server considers any non-received gradient to be $\boldsymbol{0}$. After gathering the submitted gradients, the server uses a deterministic GAR, denoted by $F$, to compute the resulting aggregate gradient $G_{agg}^{t} = F(\gradient{1}{t}, ..., \gradient{n}{t})$. In the \textbf{honest} scenario with no Byzantine workers, $F$ is simply the averaging function, i.e., $G_{agg}^{t} =\frac{1}{n}\sum\limits_{i=1}^n\gradient{i}{t}$. Then, the server uses the aggregate gradient to update the model parameters
\begin{equation}
    \label{eqn:SGD} 
    \weight{t+1} = \weight{t} - \gamma_t G_{agg}^t
\end{equation} 
where $\gamma_t$ is the learning rate in step $t$. Finally, the server broadcasts the new parameter vector $\weight{t+1}$ to all workers. 

\begin{remark}
It's important to note that only integrity and authentication are guaranteed on the communication channels between workers and server. Gradients are shared in the clear in the network.
\end{remark}

\subsection{Byzantine SGD}\label{byzRes}
In the presence of Byzantine workers sending \textbf{arbitrary} gradients, averaging the received gradients at the parameter server cannot be used to update the model parameters (Eq.~\ref{eqn:SGD}) anymore. In fact, Blanchard et al. \cite{krum} prove that any aggregation rule based on a linear combination of the received gradients is not robust in a parameter server model containing at least one Byzantine node. In order to enable the underlying model to converge, it is crucial to use a GAR that is robust to Byzantine gradients. Blanchard et al. formulate desirable robustness properties of GARs by introducing the notion of $(\alpha,f)$-Byzantine resilience \cite{krum} that can tolerate up to $f$ Byzantine nodes in the system.\\

\noindent
\paragraph{$\bm{(\alpha,f)}$\textbf{-Byzantine Resilience}}
Let $\alpha \in [0, \frac{\pi}{2}[$ be an angle, and let $n$ be the total number of nodes in the system. Let $0 \leq f \leq n$ be an upper bound on the number of Byzantine nodes. Let $\gradient{i}{t}$ denote the gradient worker $\worker{}{i}$ sends to the server at step $t$. If \worker{}{i} is honest (i.e., non-Byzantine), then $\gradient{i}{t} \sim G_t$ where $\E\left[ G_t\right]$ is equal to the true gradient $\nabla Q(\weight{t})$, i.e., $\E \left[G_t\right] = \nabla Q(\weight{t})$. Otherwise, if \worker{}{i} is Byzantine, then $\gradient{i}{t}$ can be an arbitrary vector.\\

\noindent
$F$ is said to be $(\alpha, f)$-Byzantine resilient if for any input gradients the output $R_t = F(\gradient{1}{t}, \dots , \gradient{n}{t})$ satisfies:
\begin{enumerate}\setlength\itemsep{1em}
    \item $\langle \E \left[R_t \right], \nabla Q(\weight{t})\rangle \geq (1-\sin\alpha)||\nabla Q(\weight{t})||^2 > 0$
    \item For $r \in \{2,3,4\}, \E\left[||R_t||^r\right]$ is upper bounded by a linear combination of the terms $\E\left[||G_t||^{r_1}\right], \dots , \E\left[||G_t||^{r_k}\right]$, where $r_1 + ... + r_k = r$
\end{enumerate}
Condition (1) ensures that the vector angle between the true gradient $\nabla Q(\weight{t})$ and the expected output $R_t$ of the GAR is sufficiently acute. This implies a positive lower bound on the scalar product between $\nabla Q(\weight{t})$ and $\E\left[R_t\right]$. Condition (2) ensures the boundedness of the $2^{nd}$, $3^{rd}$, and $4^{th}$ order moments of the GAR output $R_t$, which is generally required to formally guarantee the convergence of the underlying SGD algorithm~\cite{bottou, krum}. 
Some of the most popular $(\alpha,f)$-Byzantine resilient GARs include \brute~\cite{MDA}, \krum{}~\cite{krum}, \bulyan{}~\cite{brute_bulyan}, \median{}~\cite{median}, \meamed{}~\cite{meamed}, \phocas{}~\cite{phocas}, and \trimmed{}~\cite{median}. Note that this definition of Byzantine resilience also applies when the gradients sent by the workers are noisy (see Eq.~\eqref{noisyGrads}).\\

\noindent
\paragraph{\textbf{VN Ratio Condition}}
A sufficient condition for an aggregation rule $F$ to guarantee $(\alpha,f)$-Byzantine resilience is for the standard deviation of the submitted gradients to be smaller than the norm of the true gradient scaled by a multiplicative constant $k_F(n,f)$ that depends on $F$ \cite{krum, brute_bulyan}. Precisely, the following inequality must hold:
\begin{equation}\label{VN_Ratio_Original}
\frac{\sqrt{\expect{\norm{G_t-\expect{G_t}}^2}}}{\norm{\expect{G_t}}} \leq k_F(n,f)
\end{equation}
Hereafter, we refer to the left hand side of this inequality as the variance-to-norm ratio (\textbf{VN ratio}). Accordingly, we call Eq.~\eqref{VN_Ratio_Original} the \textbf{VN ratio condition}. This condition can be used even in non-trivial learning problems (e.g., when loss function $Q$ is non-convex) and has been extensively studied in the context of Byzantine-resilient SGD \cite{krum, median, brute_bulyan}. While the VN ratio condition is only \textit{sufficient}, due to the lack of necessary conditions in the literature, it remains the only existing technique to theoretically test for ($\alpha, f$)-Byzantine resilience of GARs.

\begin{remark}
Our results in Sections~\ref{Incompatibility} and~\ref{sub:str_cvx} apply for any so-called \textit{statistically-robust} GAR, i.e., a GAR which solely uses present (and/or past) submitted gradients to filter out potential attacks.
The literature also proposes two other families of GARs: \textit{suspicion-based} \cite{suspicion} and \textit{redundancy-based} \cite{redundancy}.
The \textit{suspicion-based} family relies on the parameter server having access to a possibly large training set sampled from the same distribution as the workers. This family of GARs is not private and thus, outside the scope of this paper.
The \textit{redundancy-based} family requires the honest workers to share the exact same training set, and the server to use coding theory in order to detect and filter out Byzantine gradients.
This does not fit our model as it is unclear whether these coding techniques remain applicable when honest workers also add their own privacy noise.
\end{remark}

\subsection{Differential Privacy (DP)}\label{Gaussian}
\noindent
\paragraph{\textbf{Privacy definition}} 
DP~\cite{dwork2014algorithmic} is a gold standard notion for privacy-preserving data analysis. In the context of machine learning, it guarantees that a randomized algorithm computed on a batch of training samples gives a statistically indistinguishable result when computed on an adjacent batch. In this work, we consider two data batches $\xi$ and $\xi'$ to be adjacent, denoted by $\xi \sim \xi'$, if they differ by at most one sample.\\
More precisely, we consider $\epsilon > 0$ and $\delta \in [0,1]$. A randomized algorithm \mechanismNoArg~is $(\epsilon,\delta)$-differentially private if for any adjacent batches $\xi$ and $\xi'$ and any possible set of outputs $O$, the following holds:
\begin{equation}
P[\mechanism{\xi} \in O] \leq e^{\epsilon} \times P \left[ \mechanism{\xi'} \in O \right] + \delta.
\end{equation}
The privacy budget $(\epsilon,\delta)$ measures the amount of privacy the mechanism holds. $\epsilon$ controls the privacy/utility trade-off. This means that smaller values of $\epsilon$ ensure a higher level of privacy, but usually hurt the accuracy of the model. $\delta$ can be seen as a \textit{failure} parameter. It is actually the parameter that controls the approximation one allows when enforcing DP. The smaller $\delta$, the stronger the privacy definition. In fact, $\delta$ has to be cryptographically small for $(\epsilon,\delta)$-DP to offer strong privacy guarantees \cite[Section 2]{dwork2014algorithmic}. 

In the context of distributed learning, DP can be achieved by assuming the existence of a trusted aggregation server that gathers the private information of every worker and releases a sanitized (perturbed) version of the output. However, assuming the existence of a third-party server that every worker trusts with their data is genuinely impractical. In this work, we study a more realistic setting where the server is \textit{honest-but-curious}. As stated in the introduction, the server cannot be trusted to ensure the privacy of the training datasets belonging to the nodes, but computes the aggregation of gradients correctly. In this context, the responsibility to guarantee the privacy of the \textit{personal} databases is delegated to the workers. Specifically, this means that every worker $W_i$ designs its own local randomizer $\mechanismNoArg_i$ to send a perturbed version of its gradient to the untrusted server. We call the distributed system $(\epsilon,\delta)$-differentially private if every local randomizer is $(\epsilon,\delta)$-differentially private.\\

\noindent
\paragraph{The Gaussian Mechanism \& Noisy Gradients}
One of the most common approaches to build a differentially private algorithm is to use Gaussian noise injection~\cite{DP}. This scheme is called the Gaussian mechanism. The application of this mechanism, in the context of distributed SGD with an honest-but-curious server, is to inject Gaussian noise to the gradients computed by the different workers before sending them to the parameter server.
Specifically, let $\batch{i}{t} \triangleq \left\lbrace \datapoint{i,1}{t} \ldotexp \datapoint{i,b}{t} \right\rbrace$ be the batch sampled by honest worker \worker{}{i} in step $t$. Let also $h$ be the function that computes the gradient to be sent to the server, prior to injecting the DP noise,
\begin{equation}\label{eqn:grad_i}
    h: \batch{i}{t} \rightarrow \gradient{i}{t} = \frac{1}{b} \sum\limits_{j = 1}^{b}{\compgrad{\weight{t}}{\datapoint{i,j}{t}}}.~\footnote{$Q(w)$ actually represents the expected cost function evaluated at $w$, where the expectation is taken over $\mathcal{D}$. But, by abuse of notation, we also denote by $Q(w,x)$ the empirical counterpart of this cost function computed on a given sample $x \sim \mathcal{D}$.}
\end{equation}
Then, the amount of noise required to ensure that this sanitizing strategy guarantees a given privacy budget $(\epsilon,\delta)$ depends on the the sensitivity of the mapping $h$ defined as follows, \begin{equation}\label{sensitivity}
\Delta h = \max\limits_{\xi_1 \sim \xi_2}{\normtwo{h(\xi_1) - h(\xi_2)}}. 
\end{equation}
In particular, if we assume that all gradients in our model are bounded in L2-norm by $G_{max}>0$, the sensitivity $\Delta h$ of the gradient function is upper-bounded by $\frac{2G_{max}}{b}$. We can then show that the Gaussian mechanism satisfies DP. Formally, let us consider $(\epsilon, \delta) \in (0,1)^2$ and denote $s= \frac{2 G_{max}\sqrt{2 \log(1.25/\delta)}}{b \epsilon}$. Then, the algorithm 
\begin{equation}\label{DP_Gauss}
    \mechanismNoArg: \batch{i}{t} \rightarrow  \noisygradient{i}{t} = h(\batch{i}{t}) + \boldsymbol{\noise{i}{t}} \text{ with } \boldsymbol{\noise{i}{t}} \sim \mathcal{N}\left(0, I_d \times s^2\right)
\end{equation}
is $(\epsilon,\delta)$-differentially private (see \cite[Appendix A]{dwork2014algorithmic} for a detailed proof of this statement).
Accordingly, for any step $t$, the procedure that consists in every \worker{}{i} sending the noisy gradient \begin{equation}\label{noisyGrads}
    \noisygradient{i}{t} = \gradient{i}{t} + \noise{i}{t}
\end{equation} to the parameter server is $(\epsilon, \delta)$-differentially private.
Finally, given a per-step privacy budget $(\epsilon,\delta)$, we can rely on the composition property of DP to determine the privacy guarantees of the overall learning procedure. For example, the classical composition theorem~\cite{dwork2014algorithmic} states that the DP guarantees of sequential steps add up linearly. Note also that to obtain better DP guarantees in the context of SGD, we can call for more refined tools, such as the moments accountant~\cite{abadi2016deep}. However, in this work, we are mainly interested in the impact of the \textbf{per-step} privacy budget on the robustness of the system to Byzantine nodes.

\begin{remark} Before going further, let us highlight a few points. 
\begin{itemize}
    \item Although we focus on Gaussian noise injection in this work, our results in Sections~\ref{Incompatibility} and~\ref{sub:str_cvx} can easily be adapted to any other DP mechanism based on noise injection (e.g., the Laplacian mechanism~\cite{DP}). Therefore, in the remainder of this work, we no longer mention the Gaussian mechanism when talking about DP noise injection.
    \item Note that the Gaussian mechanism assumes ($\epsilon, \delta$) to be in $(0,1)^2$, which is typically the per-step privacy budget considered in the literature~\cite{abadi2016deep, bassily, papernot} (even in challenging tasks such as deep learning). Therefore, in the remainder of this paper, we assume w.l.o.g that $\epsilon < 1$ and $\delta <1$.
\end{itemize}
\end{remark}

\section[Incompatibility of Byzantine Resilience and DP in SGD (General Setting)]{Incompatibility of $(\alpha,f)$-Byzantine Resilience and DP in SGD (General Setting)}\label{Incompatibility}
First, we study the framework for distributed SGD in the general setting, i.e., without making any assumption on the convexity of the objective function $Q$. In order for the DP noise to be well calibrated, we assume the norm of gradients to be bounded. This assumption can be typically enforced via gradient clipping \cite{clipping} methods.
\begin{assumption}[Bounded norm]
\label{asp:bnd_norm}
There exists a real value $G_{max} >0$ such that for any learning step $t$ and any parameter vector $w_t$, $\norm{\nabla Q(w_t)} \leq G_{max}$.
\end{assumption}

\noindent
In this context, as discussed in Section~\ref{byzRes}, the only \textbf{existing} method to guarantee $(\alpha,f)$-Byzantine resilience for an aggregation rule $F$ is verifying the VN ratio condition. Hence, combining DP with Byzantine resilience requires the VN ratio to still be upper bounded by $k_F(n,f)$ after injecting the noise to the gradients. Precisely, to guarantee that a gradient aggregation step is \textit{both} $(\epsilon,\delta)$-differentially private and $(\alpha,f)$-Byzantine resilient, the following must hold:
\begin{equation}\label{VN_Ratio_Condition}
\frac{\sqrt{\expect{\norm{G_t-\expect{G_t}}^2} + 8d\frac{ G_{max}^2}{\epsilon^2 b^2}\log\left(\frac{1.25}{\delta}\right)}}{\norm{\expect{G_t}}} \leq k_F(n,f), 
\end{equation}
where $d$ is the size of the model, $b$ is the batch size, and $(\epsilon, \delta)$ is the per-step privacy budget. Eq.~\eqref{VN_Ratio_Condition} is the direct adaptation of Eq.~\eqref{VN_Ratio_Original} after accounting for the variance of the Gaussian noise in the numerator of the VN ratio (refer to Section \ref{Gaussian}). This adaptation simply comes from decomposing the variance of \noisygradient{i}{t} into the sum of the variances of the independent random variables \gradient{i}{t} and \noise{i}{t} respectively.\\

\noindent
In the high privacy regime, i.e., when both $\epsilon$ and $\delta$ approach $0$, we see that the VN ratio increases, making it more difficult to satisfy Eq.~\ref{VN_Ratio_Condition} and thus, ensure $(\alpha, f)$-Byzantine resilience. In fact, we show that the VN ratio condition can only hold {\it if} either the batch size $b$ is in $\Omega\left( \sqrt{d} \right)$ or the proportion of Byzantine nodes in the system $\frac{f}{n}$ is in $O\left( \frac{1}{\sqrt{d}}\right)$, depending on the GAR used as shown in Table~\ref{table:SummaryGARs}. Below, we demonstrate this statement by focusing on one popular GAR called MDA \cite{MDA} for space limitation, but similar results also hold for the other GARs from the literature, as illustrated in Table~\ref{table:SummaryGARs} (detailed proofs are deferred to the Appendix).

\begin{proposition} \label{th:MDAgeneral}
Let $(\epsilon,\delta) \in (0,1)^2$ be the constant privacy budget used to inject DP noise to the gradients and $F=\brute{}$. Then, the VN ratio condition can only hold if $\frac{f}{n} \in O\left(\frac{b}{\sqrt{d} + b}\right)$.
\end{proposition}

\begin{sketchofproof}
We reason by contraposition to show that $\frac{f}{n} \in O\left( \frac{b}{\sqrt{d} + b}\right)$ is necessary for the VN ratio condition to hold. This means that we will actually find a sufficient condition for it \textit{not} to hold. 
First, notice that regardless of the GAR used, the VN ratio condition does not hold if  
$$\frac{\sqrt{\expect{\norm{G_t-\expect{G_t}}^2} + 8d\frac{G_{max}^2}{\epsilon^2 b^2}\log\left(\frac{1.25}{\delta}\right)}}{\norm{\expect{G_t}}} > k_F(n,f).$$
Since $\expect{\norm{G_t-\expect{G_t}}^2} \geq 0$, the above inequality holds if
\begin{equation*}
    \norm{\expect{G_t}} < G_{max} \times \sqrt{ 8d\frac{\log\left(\frac{1.25}{\delta}\right)}{\epsilon^2b^2}} \times \frac{1}{ k_F(n,f)}.
\end{equation*}

Let us denote $C= \epsilon / \sqrt{\log\left(\frac{1.25}{\delta}\right)}$. Since $G_{max}$ is the maximal L2 norm gradients can have and $\expect{G_t} = \nabla Q(w_t)$, the VN ratio condition does not hold whenever $$\sqrt{\frac{8d}{C^2 b^2}} \times \frac{1}{ k_F(n,f)} > 1 \Leftrightarrow \frac{1}{ k_F(n,f)} > \frac{b \times C}{\sqrt{8d}}.$$
Now, notice that since $F=\brute{}$, we have $k_F(n,f)=(n-f)/(\sqrt{8}f)$. If we denote by $\tau = \frac{f}{n}$ the proportion of Byzantine workers in the system, we have $k_F(n,f)=(1-\tau)/(\sqrt{8}\tau)$. Then the above inequality can be rewritten as $8 \sqrt{d} \tau > (1-\tau) C \times b$, which is equivalent to saying that $\tau >  (C \times b) / (8\sqrt{d} +  C \times b)$.\\

\noindent
Finally, since the privacy budget $(\epsilon, \delta) \in (0,1)^2$, $C$ is a negligible constant w.r.t $b$ and $d$. Then, when computing a differentially private gradient step, the VN ratio cannot hold unless $\frac{f}{n} \in O\left(\frac{b}{\sqrt{d} + b}\right)$.\qed
\end{sketchofproof}

\noindent
Proposition~\ref{th:MDAgeneral} implies that ensuring both DP and $(\alpha,f)$-Byzantine resilience with \brute{} requires the batch size $b$ to grow linearly with the square root of the number of parameters ($\sqrt{d}$), or for the fraction of Byzantine nodes $\frac{f}{n}$ to converge towards $0$ at the rate of $\sqrt{d}$. For example, if we consider the ResNet-50 model \cite[Table 8]{DBLP:journals/corr/ZagoruykoK16} where $d = 25.6 \times 10^6$ parameters, then the we need a batch size $b > 5000$, which is clearly impractical. In this context, we can only ensure protection against Byzantine nodes in the trivial setting where almost all workers are honest. Although we only state this incompatibility result for \brute{}, similar results can be demonstrated for any GAR from the literature using the same proof scheme. Table~\ref{table:SummaryGARs} summarizes the necessary conditions for the VN ratio condition (Eq.~\eqref{VN_Ratio_Condition}) to hold depending on the GAR at hand.

\begin{table}[ht]
\caption{Necessary condition for the VN ratio condition to hold in the context of $(\epsilon, \delta)$-DP distributed SGD, for 7 different GARs form the literature.}
\centering
\begin{tabular}{|c|c|c|c|}
\hline
\begin{tabular}{@{}c@{}}\textbf{\krum{}}, \textbf{\median{}}, \\ \textbf{\bulyan{}}, \textbf{\meamed{}}\end{tabular} & \textbf{\brute{}} & \textbf{\phocas{}}, \textbf{\trimmed{}}\\
\hline
 $b \in \Omega(\sqrt{nd})$ & $\frac{f}{n} \in O\left(\frac{b}{\sqrt{d}+b}\right)$ & $\frac{f}{n} \in O\left(\frac{b^2}{d+b^2}\right)$\\
\hline
\end{tabular}
\label{table:SummaryGARs}
\end{table}

\noindent
Next, to get a more fine-grained analysis on the impact of noise injection on the $(\alpha, f)$-Byzantine resilience property of the distributed SGD scheme, we study the particular case where the loss function is strongly convex.

\input{strongly_convex}
\input{experiments}

\section{Related Work}\label{relatedWork}
There has been several efforts to implement the SGD algorithm in a way to protect the privacy of the training data. The problem was tackled both in the centralized setting \cite{DP_SGD1, DP_SGD2, abadi2016deep} and in the federated learning setting with multiple workers collaborating to train an aggregate model \cite{DP_SGD_Fed1, DP_SGD_Fed2}. The idea mainly consists in adding DP noise to the gradients computed by the different workers, as explained in Section \ref{Gaussian}. Other techniques have investigated encrypting the gradients shared in the distributed network \cite{HEGrads, encryptGrads} to prevent a passive attacker from violating the privacy of data nodes by simply intercepting the gradients exchanged. However, the Byzantine resilience aspect of distributed SGD is outside the scope of these works: the presented solutions do not account for Byzantine data nodes that can disrupt the training by sending erroneous gradients.\\

\noindent
A different line of research has been devoted to developing Byzantine resilience schemes for federated learning, i.e., designing GARs that are robust to a certain fraction of the data nodes being Byzantine \cite{krum, brute_bulyan, median, MDA, meamed, phocas}. However, these papers do not consider the privacy threat associated with sharing gradients in the clear among participants.\\

\noindent
Recently, some works tried to simultaneously mitigate both threats. Hi et al. propose a Byzantine-resilient and privacy-preserving solution~\cite{twoServer_Jaggi}. However, the authors adopt a weaker threat model than ours by assuming the presence of two non-colluding and honest-but-curious servers, which is a stronger assumption than the single-server solution. Furthermore, the authors use additive secret sharing to protect the privacy of the data, which provide weaker guarantees than DP. We also mention \brea{} \citep{brea}, that is a single-server approach but does not use DP either. 
The approach that might be the most related to our work is \textit{LearningChain} \cite{LearningChain} since it seems to be the only other framework that combines DP and Byzantine resilience. Although Chen et al. claim they solve this problem, \textit{LearningChain} remains an experimental method. In fact, the authors do not provide any theoretical guarantees on the Byzantine resilience or convergence of their $l$-nearest aggregation algorithm \cite{LearningChain}, as done by other works also constructing Byzantine-resilient GARs \cite{krum, brute_bulyan}.

\section{Concluding Remarks}
This paper provides the first theoretical analysis on the problem of combining DP and $(\alpha, f)$-Byzantine resilience in distributed SGD frameworks. Our theoretical and experimental findings show that the problem is indeed challenging and that the classical noise injection techniques for DP make the Byzantine resilience of the SGD algorithm depend \textbf{unfavorably} on the size of the model. Combining these two concepts for large models such as neural networks is thus impractical, and potentially requires designing of alternate techniques, be it for DP or Byzantine resilience. For instance, having observed that DP noise makes the variance of the gradients grow linearly with size $d$ of the ML model, it would be interesting to study whether {\em variance reduction} techniques~\cite[Section 5]{bottou2018optimization}, such as {\em dynamic sampling} or {\em exponential gradient averaging}, can alleviate this dependence on model size.\\

\noindent
An alternate future direction for this work is to design differentially private schemes that depend less on the size of the ML model. For example, we could complement classical noise injection techniques with cryptographic primitives to provably reduce the variance of the injected noise while still ensuring $(\epsilon, \delta)$-DP. We could also investigate shuffling techniques for privacy amplification \cite{amplification}. Furthermore, our work mainly focuses on statistically-robust GARs. As such, other families of $(\alpha, f)$-Byzantine resilient GARs, such as \textit{suspicion-based} \cite{suspicion} and \textit{redundancy-based} \cite{redundancy} GARs, do not comply with the setting we considered for privacy reasons. It would be interesting to study whether they could be adapted to meet strong privacy requirements in the context of distributed SGD.

\begin{acks}
John Stephan and Sébastien Rouault have been supported in part by the Swiss National Science Foundation (FNS grant N°$200021\_182542$). Rafaël Pinot has been supported in part by Ecocloud, an EPFL research center (Postdoctoral Research Award).
\end{acks}

\appendix

\section{Impracticality Results and Proofs}
\paragraph{Proof of Proposition 1.}
\begin{proof}
We reason by contraposition to show that $\frac{f}{n} \in O\left( \frac{b}{\sqrt{d} + b}\right)$ is necessary for the VN ratio condition to hold. This means that we will actually find a sufficient condition for it \textit{not} to hold. 
First, notice that regardless of the GAR we use, the VN ratio condition does not hold if $$\frac{\sqrt{\expect{\norm{G_t-\expect{G_t}}^2} + 8d\frac{G_{max}^2}{\epsilon^2 b^2}\log\left(\frac{1.25}{\delta}\right)}}{\norm{\expect{G_t}}} > k_F(n,f).$$
In particular, since $\expect{\norm{G_t-\expect{G_t}}^2} \geq 0$, the above inequality hold as soon as 
$$\frac{\sqrt{ 8d\frac{ G_{max}^2}{\epsilon^2 b^2}\log\left(\frac{1.25}{\delta}\right)}}{\norm{\expect{G_t}}} > k_F(n,f).$$
With some rewriting, we get the following inequality
\begin{equation*} 
    \norm{\expect{G_t}} < G_{max} \times \sqrt{ 8d\frac{\log\left(\frac{1.25}{\delta}\right)}{\epsilon^2b^2}} \times \frac{1}{ k_F(n,f)}.
\end{equation*} 
Let us denote $C= \epsilon / \sqrt{\log\left(1.25/\delta\right)}$. Since $G_{max}$ is the maximal L2 norm the gradient can take and $\expect{G_t} = \nabla Q(w_t)$, the VN ratio condition does not hold whenever 
\begin{equation}\label{eq:failurecondition}
    \sqrt{\frac{8d}{C^2 b^2}} \times \frac{1}{ k_F(n,f)} > 1 \Leftrightarrow \frac{1}{ k_F(n,f)} > \frac{b \times C}{\sqrt{8d}}.
\end{equation}
Now, notice that since $F=\brute{}$, we have $k_F(n,f)=(n-f)/(\sqrt{8}f)$. If we denote by $\tau = \frac{f}{n}$ the proportion of Byzantine workers in the system, we have $k_F(n,f)=(1-\tau)/(\sqrt{8}\tau)$. Then the above inequality can be rewritten as $8 \sqrt{d} \tau > (1-\tau) C \times b$, which is equivalent to saying that $\tau >  (C \times b) / (8\sqrt{d} +  C \times b)$. 

Finally, since $(\epsilon, \delta) \in (0,1)^2$, $C$ is a negligible constant w.r.t $b$ and $d$. Then, when computing a differentially private gradient step, the VN ratio cannot hold unless $\frac{f}{n}$ is in $O\left(\frac{b}{\sqrt{d} + b}\right)$.
\end{proof}

\begin{proposition} \label{prop:ImpossibilityKrum} Let $(\epsilon,\delta) \in (0,1)^2$ be the constant privacy budget used to inject DP noise to the gradients and $F \in \{$ \krum{}, \bulyan{}, \median{}, \meamed{}$\}$. Then, the VN ratio condition can only hold if $b \in \Omega(\sqrt{n \times d})$.
\end{proposition}

\noindent
\begin{proof}
Recall that, thanks to Eq.~\eqref{eq:failurecondition} (see proof of Proposition~\ref{th:MDAgeneral}), the VN ratio condition cannot hold if  $$ \sqrt{\frac{8d}{C^2 b^2}} \times \frac{1}{ k_F(n,f)} > 1 \Leftrightarrow \frac{1}{ k_F(n,f)} > \frac{b \times C}{\sqrt{8d}},$$ with $C= \epsilon / \sqrt{\log\left(1.25/\delta\right)}$.

\textbf{$\bullet$ If $F \in \{ \krum{}, \bulyan{} \} $ } we have $k_F(n,f)= 1/\sqrt{2\eta(n,f)}$ with $\eta(n,f)=n-f + \frac{ f (n-f-2) + f^2(n-f-1)}{n-2f -2}$. Hence, Eq.~\eqref{eq:failurecondition} holds whenever $ \sqrt{2\eta(n,f)} > C \times b /\sqrt{8d}.$ In particular, since $\eta(n,f) > n + f^2$, the above holds when $\sqrt{ 2(n + f^2) }  > C \times b /\sqrt{8d} \Leftrightarrow \sqrt{ 16d(n + f^2) }  > C \times b$.

\textbf{$\bullet$ If $F = \median{}$}, we have $k_F(n,f)= 1 / \sqrt{n-f} $, with the additional assumption $2f \leq n-1$. Following the same steps as above, Eq.~\eqref{eq:failurecondition} holds whenever $\sqrt{n-f} > C \times b /\sqrt{8d}.$ Since $f \leq (n-1)/2$, the above inequality holds in particular when $\sqrt{ (n+1)/2} > C \times b /\sqrt{8d} \Leftrightarrow \sqrt{4d(n+1)} > C \times b.$ 

\textbf{$\bullet$ If $F = \textbf{\meamed{}}$}, we have $k_F(n,f)= 1 /\sqrt{10(n-f)}$, also with the additional assumption $2f \leq n-1$. Following the same logic as for \median{},  Eq.~\eqref{eq:failurecondition} holds whenever $\sqrt{40d(n+1)} > C \times b.$ 

\noindent
Finally, since $(\epsilon, \delta) \in (0,1)^2$, $C$ is a negligible constant w.r.t $b$ and $d$. Then, for $F \in \{ \krum{},\bulyan{},\median{},\meamed{}\}$, when computing a differentially private gradient step, the VN ratio cannot hold unless the batch size $b$ is in $\Omega(\sqrt{n \times d})$.
\end{proof}

\begin{proposition} \label{prop:ImpossibilityPhocas} Let $(\epsilon,\delta) \in (0,1)^2$ be the constant privacy budget used to inject DP noise to the gradients and $F \in \{$ \trimmed{}, \phocas{}$\}$. Then, the VN ratio condition can only hold if $\frac{f}{n} \in O\left(\frac{b^2}{d+b^2}\right)$.  
\end{proposition}

\begin{proof}
Recall that, thanks to Eq.~\eqref{eq:failurecondition} (see proof of Proposition~\ref{th:MDAgeneral}), the VN ratio condition cannot hold if  $$ \sqrt{\frac{8d}{C^2 b^2}} \times \frac{1}{ k_F(n,f)} > 1 \Leftrightarrow \frac{1}{ k_F(n,f)} > \frac{b \times C}{\sqrt{8d}},$$ with $C= \epsilon / \sqrt{\log\left(1.25/\delta\right)}$.

\textbf{$\bullet$ If $F = \textbf{\trimmed{}}$}, we have $k_F(n,f)=\sqrt{\frac{(n-2f)^2}{2(f+1)(n-f)}}$. Then, thanks to Eq.~\eqref{eq:failurecondition} (see proof of Proposition~\ref{th:MDAgeneral}), the VN ratio condition does not hold whenever $ \sqrt{\frac{2(f+1)(n-f)}{(n-2f)^2}} > C \times b / \sqrt{8d}.$  
In particular, since $n-f \geq n-2f$, this inequality holds as soon as $ \frac{2f}{(n-2f)}> C^2 \times b^2 /8d.$ Let us denote by $\tau = \frac{f}{n}$ the proportion of Byzantine workers in the system. We have that $\frac{2f}{n-2f} = \frac{2\tau}{1-2\tau}$. Then the above inequality can be rewritten as $16 \tau d  > (1-2\tau) C^2 b^2$. This is equivalent to saying that $\tau > \frac{ C^2 b^2}{16 d +  2 C^2 b^2}$.

\textbf{$\bullet$ If $F = \textbf{\phocas{}}$}, we have $k_F(n,f)=\sqrt{ 4 + \frac{(n-2f)^2}{12(f+1)(n-f)}}$. Then, following the same logic as for \textbf{\trimmed{}}, Eq.~\eqref{eq:failurecondition} hold whenever
$\tau > \frac{C^2 b^2 }{64 d + 2 C^2 b^2}.$ 

Finally, since $(\epsilon, \delta) \in (0,1)^2$, $C$ is a negligible constant w.r.t $b$ and $d$. Then, for $F \in \{ \trimmed{}, \phocas{}\}$, when computing a differentially private gradient step, the VN ratio cannot hold unless $\frac{f}{n}$ is in $O\left(\frac{b^2}{d+b^2}\right)$.
\end{proof}

\input{app_priv-res-tf}


\bibliographystyle{ACM-Reference-Format}
\bibliography{ref}

\end{document}

%% file: common.tex
\newcommand{\papertitle}{Differential Privacy and Byzantine Resilience in SGD:\\ Do They Add Up?}
\newcommand{\paragraphspaceheight}{0.2cm}
\newcommand{\paragraphspace}{\vspace{\paragraphspaceheight}}

\newcommand{\ldotexp}{, \thinspace \ldots{} \thinspace, \thinspace}
\newcommand{\setr}{\mathbb{R}}

\newcommand{\expect}[1]{\mathop{{}\mathbb{E}}\left[{#1}\right]}

\newcommand{\suchthat}{\ensuremath{~\middle|~}}
\newcommand{\knowing}{\suchthat{}}

\newcommand{\normtwo}[1]{\left\lVert{#1}\right\rVert_{2}}

\newcommand{\norm}[1]{\left\lVert{#1}\right\rVert}

\newcommand{\indexvar}[3]{\ensuremath{{{#3}^{\ifthenelse{\equal{#1}{}}{}{\left({#1}\right)}}_{#2}}}}
\newcommand{\indexvarNoPar}[3]{\ensuremath{{{#3}^{\ifthenelse{\equal{#1}{}}{}{\left{#1}\right}}_{#2}}}}

\newcommand{\params}[2]{\indexvarNoPar{#1}{#2}{w}}

\newcommand{\noise}[2]{\indexvar{#1}{#2}{\boldsymbol{y}}}
\newcommand{\batch}[2]{\indexvar{#1}{#2}{\xi}}
\newcommand{\datapoint}[2]{\indexvar{#1}{#2}{x}}
\newcommand{\mechanismNoArg}[0]{\ensuremath{\mathcal{M}}}
\newcommand{\mechanism}[1]{\ensuremath{\mathcal{M}({#1})}}
\newcommand{\weight}[1]{\params{}{#1}}

\newcommand{\lossname}{\ensuremath{Q}}
\newcommand{\realloss}[1]{\ensuremath{\lossname{}\left({#1}\right)}}

\newcommand{\compgrad}[2]{\ensuremath{\nabla{}\lossname{}\left({#1}, {#2}\right)}}

\newcommand{\gradient}[2]{\indexvar{#1}{#2}{g}}
\newcommand{\attack}[1]{\indexvar{}{#1}{a}}

\newcommand{\worker}[2]{\indexvar{#1}{#2}{W}}

\newcommand{\noisygradient}[2]{\indexvar{#1}{#2}{o}}

\newcommand{\proba}[2]{\ensuremath{\text{P}\!\left({#1}\ifthenelse{\equal{#2}{}}{}{\knowing{}{#2}}\right)}}

\newcommand{\brute}{\textit{MDA}}
\newcommand{\krum}{\textit{Krum}}
\newcommand{\bulyan}{\textit{Bulyan}}
\newcommand{\median}{\textit{Median}}
\newcommand{\meamed}{\textit{Meamed}}
\newcommand{\phocas}{\textit{Phocas}}
\newcommand{\trimmed}{\textit{Trimmed Mean}}
\newcommand{\brea}{\textit{BREA}}

\newcommand{\includeplot}[2]{\includegraphics[width=#1,trim={6mm 6mm 6mm 6mm},clip]{plots/svm-phishing-simples-logit-#2.png}}

\renewcommand{\paragraph}[1]{\textbf{#1}~}

%% file: strongly_convex.tex
\section{Learning with Byzantine Resilience and DP: The Case of Strong-Convexity}\label{sub:str_cvx}

In this section, we take a step forward and show that regardless of the GAR used, the training error rate deteriorates in the presence of Byzantine workers when injecting DP noise to the gradients, as specified in Section~\ref{Gaussian}. For simplicity of presentation, we assume the cost function $Q(w)$ to be strongly convex, the gradient $\nabla Q(w)$ to be globally Lipschitz continuous, and the stochastic gradients $\nabla Q(w, \, x)\, \vline_{x \sim \mathcal{D}}$ to have bounded variance for all $w$. It is important to note that these assumptions hold true in many distributed learning problems~\cite{bottou2018optimization}. We state these assumptions formally as follows. 

\begin{assumption}[Strong convexity]
\label{asp:str_cvx}
There exists a finite real value $\lambda > 0$ such that for all $w, \, w' \in \R^d$,
\begin{align*}
    \iprod{w - w'}{\nabla Q(w) - \nabla Q(w')} \geq \lambda \norm{w - w'}^2.
\end{align*}
\end{assumption}

\begin{assumption}[Lipschitznes]
\label{asp:lip}
There exists a finite positive real value $\mu$ such that
\begin{align*}
    \norm{\nabla Q(w) - \nabla Q(w')} \leq \mu \norm{w - w'}, \quad \forall w, \, w' \in \R^d.
\end{align*}
\end{assumption}

\begin{assumption}[Bounded variance]
\label{asp:bnd_var}
There exists a positive real value $\sigma < \infty$ such that for all $w$,
\[\E_{x \sim \mathcal{D}} \left[\norm{ \nabla Q\left(w, \, x \right) - \nabla Q\left(w\right)}^2\right] \leq \sigma^2.\]
\end{assumption}

\noindent
Under the above assumptions, we show in Theorem~\ref{thm:priv_res_tf} that even if we were to successfully design a GAR that is $(\alpha, \, f)$-Byzantine resilient whilst guaranteeing $(\epsilon, \, \delta)$-DP using noise injection, the training error rate of the distributed SGD algorithm is $O\left(d / T \right)$ where $T$ denotes the number of steps. On the other hand, we can show from Eq.~\eqref{eqn:main_no-dp-rate} below, that the same distributed SGD algorithm with $(\alpha, \, f)$-Byzantine resilience can obtain a training error rate of $O\left(1 / T \right)$ in the absence of DP. In the remainder of the section, let $\xi_t$ denote the random data points sampled by the workers in step $t$, and let $\E_{\xi_t} [\cdot]$ denote the conditional expectation with respect to $\xi_t$, given the estimate $w_{t}$. For $t \geq 1$, let $\E_t[\cdot] = \E_{\xi_1} \ldots \E_{\xi_t} [\cdot]$. Then, the following theorem holds.

\begin{theorem}
\label{thm:priv_res_tf}
Suppose that Assumptions~\ref{asp:bnd_norm},~\ref{asp:str_cvx},~\ref{asp:lip}, and~\ref{asp:bnd_var} hold true. Consider a GAR named $F: \R^{d \times n} \to \R^d$. Suppose that in each step $t$, the server updates its learning parameter $w_{t}$ using the noisy gradients $\left\{ o^{(1)}_t, \ldots, \, o^{(n)}_t \right\}$ specified in Section~\ref{Gaussian} and GAR $F$ as follows:
\begin{align}
    w_{t+1} = w_t - \gamma_t \, F\left( o^{(1)}_t, \ldots, \, o^{(n)}_t \right).  \label{eqn:update_thm}
\end{align}
Let $Q^* = \min_{w \in \R^d} Q(w)$. If $F$ is $(\alpha, \, f)$-Byzantine resilient for steps $1$ to $T$, and $\gamma_t = \frac{1}{\lambda (1 - \sin \alpha) \, t}$, then
\begin{align*}
    \E_{T} \left[ Q(w_{T+1}) - Q^* \right] \in \Theta\left( \frac{d \log(1/\delta)}{ T b^2 \epsilon^2} \right).
\end{align*}
\end{theorem}

\begin{sketchofproof}
We let $F_t \triangleq F\left( o^{(1)}_t, \ldots, \, o^{(n)}_t \right)$. The proof relies on the following key observation that holds true when $F_t$ is assumed $(\alpha, \, f)$-Byzantine resilient for steps $1$ to $T$:  
\begin{align}
    \langle \E_{\xi_t} \left[F_t \right], \, \nabla Q(w_t) \rangle & \geq (1-\sin\alpha) \norm{\nabla Q(w_t)}^2 > 0. \label{eqn:main_pi_byzres}
\end{align}
Also, there exists a real-value $c > 0$ such that for each $ t \in \{ 1, \ldots, \, T \}$,
\begin{align}
    \E_{\xi_t} \left[ \norm{F_t}^2 \right] \leq c \left(\frac{\sigma^2}{b} + d s^2 + G_{\max}^2 \right). \label{eqn:main_byrres_mom_2}
\end{align}
The details for obtaining the results in~\eqref{eqn:main_pi_byzres} and~\eqref{eqn:main_byrres_mom_2} can be found in Appendix~\ref{app:priv_res_tf}.
From~\eqref{eqn:update_thm} and Assumption~\ref{asp:lip}, we obtain for all $t$ that
\begin{align*}
    Q(w_{t+1}) \leq Q(w_t)  - \gamma_t \iprod{\nabla Q(w_t)}{ F_t } + \frac{1}{2} \mu \gamma^2_t \norm{F_t}^2 . 
\end{align*}
As $\E_{\xi_t}\left[ Q(w_t)\right] = Q(w_t)$, the above implies that, for all $t$,
\begin{multline*}
    \E_{\xi_t} \left[Q(w_{t+1}) \right] \leq Q(w_t)  - \gamma_t \iprod{\nabla Q(w_t)}{ \E_{\xi_t} \left[ F_t \right]}\\
    + \frac{1}{2} \mu \gamma^2_t \E_{\xi_t} \left[ \norm{F_t}^2 \right]. 
\end{multline*}
Substituting from~\eqref{eqn:main_pi_byzres}, and then using the fact that $\norm{\nabla Q(w_t)}^2 \geq 2 \lambda \, \left(Q(w_t) - Q^* \right)$ under Assumption~\ref{asp:str_cvx}, we obtain that
\begin{multline*}
\E_{\xi_t} \left[Q(w_{t+1})\right] - Q^* \leq \left(1 - 2 \lambda (1 - \sin \alpha) \gamma_t \right) \left( Q(w_{t}) - Q^* \right)\\ 
+ \frac{1}{2} \mu \gamma^2_t \E_{\xi_t} \left[\norm{F_t}^2 \right] , \quad \text{ for } t = 1, \ldots, \, T.
\end{multline*}
Substituting from~\eqref{eqn:main_byrres_mom_2}, and then taking expectation $\E_{\xi_1} \dots \E_{\xi_{t-1}}$ on both sides above (for $t > 1$), we obtain that
\begin{multline*}
    \E_t \left[Q(w_{t+1}) \right] - Q^* \leq \left(1 - 2 \lambda (1 - \sin \alpha) \gamma_t \right) \left( \E_{t-1} \left[Q(w_{t}) \right] - Q^* \right)\\
    + \frac{1}{2} \mu c \left(\frac{\sigma^2}{b} + d s^2 + G_{\max}^2 \right) \gamma^2_t, \quad \text{ for } t = 2, \ldots, \, T.
\end{multline*}
Thus, from recursive substitutions we obtain that
\begin{align}
    \E_T \left[ Q(w_{T+1}) \right] - Q^*  \leq \frac{1}{T + 1} \left(\frac{\mu \, c}{2 \lambda^2 (1 - \sin \alpha)^2} \right) \, \left(\frac{\sigma^2}{b} + d s^2 + G_{\max}^2 \right) \label{eqn:main_no-dp-rate}
\end{align}
Substituting from Section~\ref{Gaussian}, $s= \frac{2 G_{max}\sqrt{2 \log(1.25/\delta)}}{b \epsilon}$, implies
\[\E_T \left[ Q(w_{T+1}) \right] - Q^* \in O\left( \frac{d \log(1/\delta)}{ T b^2 \epsilon^2} \right).\]
To prove the {\bf lower bound} we consider a specific cost function $Q(w) := (1/2)\E_{x \sim \mathcal{D}}\norm{w - x}^2$ where $\mathcal{D} = \mathcal{N}\left(\bar{x}, \, \frac{\sigma^2}{d} I_d\right)$ where $\bar{x} \in \R^d$. Note that in this case, the minimum point $w^*$ of $Q(w)$ is simply $\bar{x}$. Thus, $Q^* = (1/2)\E_{x \sim \mathcal{D}}\left[\norm{\bar{x} - x}^2\right]$, and for any $w \in \R^d$, 
\[Q(w)  = \frac{1}{2}\E_{x \sim \mathcal{D}}\left[\norm{w - \bar{x} + \bar{x} - x}^2\right] = \frac{1}{2} \norm{w - \bar{x}}^2 + Q^*.\]
Now, we consider a hypothetical GAR $F$ that outputs the gradient of an honest worker in each step $t$.\footnote{In practice, such a GAR may never exist as the identity of honest workers is a priori unknown.} As the gradients of honest workers are unbiased estimators of the true gradient, this particular GAR is indeed $(\alpha, \, f)$-Byzantine resilient. The problem of computing $w^*$ using outputs of $F$, operating on noisy workers' gradients (defined in Section~\ref{Gaussian}), in $T$ steps is equivalent to estimating $\bar{x}$ using $T$ noisy observations $\{\bar{x} + z_t; ~ 1 \leq t \leq T\}$ where $z_t \sim \mathcal{N}\left(0,  \left(\frac{\sigma^2}{db} + d s^2 \right) I_d\right), \forall t$.
Hence, by the Cram\'er-Rao bound~\cite{rao1992information}, for any (stochastic) estimate $\hat{w}$ of $\bar{x}$ after $T$ steps, $\expect{\norm{\hat{w} - \bar{x}}^2} \geq \left(\frac{\sigma^2}{b} + d s^2 \right)\frac{1}{T}$. Thus,
\[\expect{Q(\hat{w})} - Q^* \geq \left(\frac{\sigma^2}{b} + d s^2 \right)\frac{1}{2T}.\]
Substituting $s$ above proves the lower bound.

$\hfill \qed$
\end{sketchofproof}

\noindent
According to Theorem~\ref{thm:priv_res_tf}, the training error rate may degrade linearly with the size $d$ of the learning model, for a fixed privacy budget $(\epsilon, \, \delta)$ and a fixed batch size $b$. Now, if $b \in \Omega{(\sqrt{d})}$, then Theorem~\ref{thm:priv_res_tf} along with the observations in Table~\ref{table:SummaryGARs} implies that the distributed SGD algorithm with DP gradients and $(\alpha, f)$-Byzantine-resilient GARs (e.g., \krum{} and \brute{}) may attain a training error rate of $O(1/T)$, which is indeed optimal~\cite[Section 4]{bottou2018optimization}. However, in many cases, $b \in \Omega{(\sqrt{d})}$ is practically \textbf{not} viable, as argued in Section~\ref{Incompatibility} through the example of ResNet-50. In the next section, we see that even for moderate model sizes, combining DP and Byzantine resilience requires impractically large batches.

%% file: experiments.tex
\section{Experimental Support}
Our theory predicts that the training batch size $b$ must grow as fast as $\sqrt{d}$ to satisfy the VN ratio condition of Byzantine-resilient GARs\footnote{Our analysis shows that such a condition exists for all known statistically-robust Byzantine-resilient GARs, when the ratio of Byzantine workers $\frac{f}{n}$ is fixed.}.
Even for small neural networks ($d \approx 10^5$), our theoretical results already suggest that unrealistically large batch sizes are required to satisfy this VN condition.

In this section, we experiment training with a smaller, convex model, measuring how the loss and accuracy respectively evolve with the privacy parameter $\epsilon$ and the training batch size $b$.
We notice that while DP and Byzantine resilience can most of the time be guaranteed alone, combining the two proves to be difficult even with the fairly small ($d = 69$) model at hand.

\begin{figure}
    \centering
    \makebox[0.5\linewidth][c]{\small{}Without privacy noise}%
    \makebox[0.5\linewidth][c]{\small{}With privacy noise ($\epsilon = 0.2$)}\\[0.1mm]%
    \includeplot{0.5\linewidth}{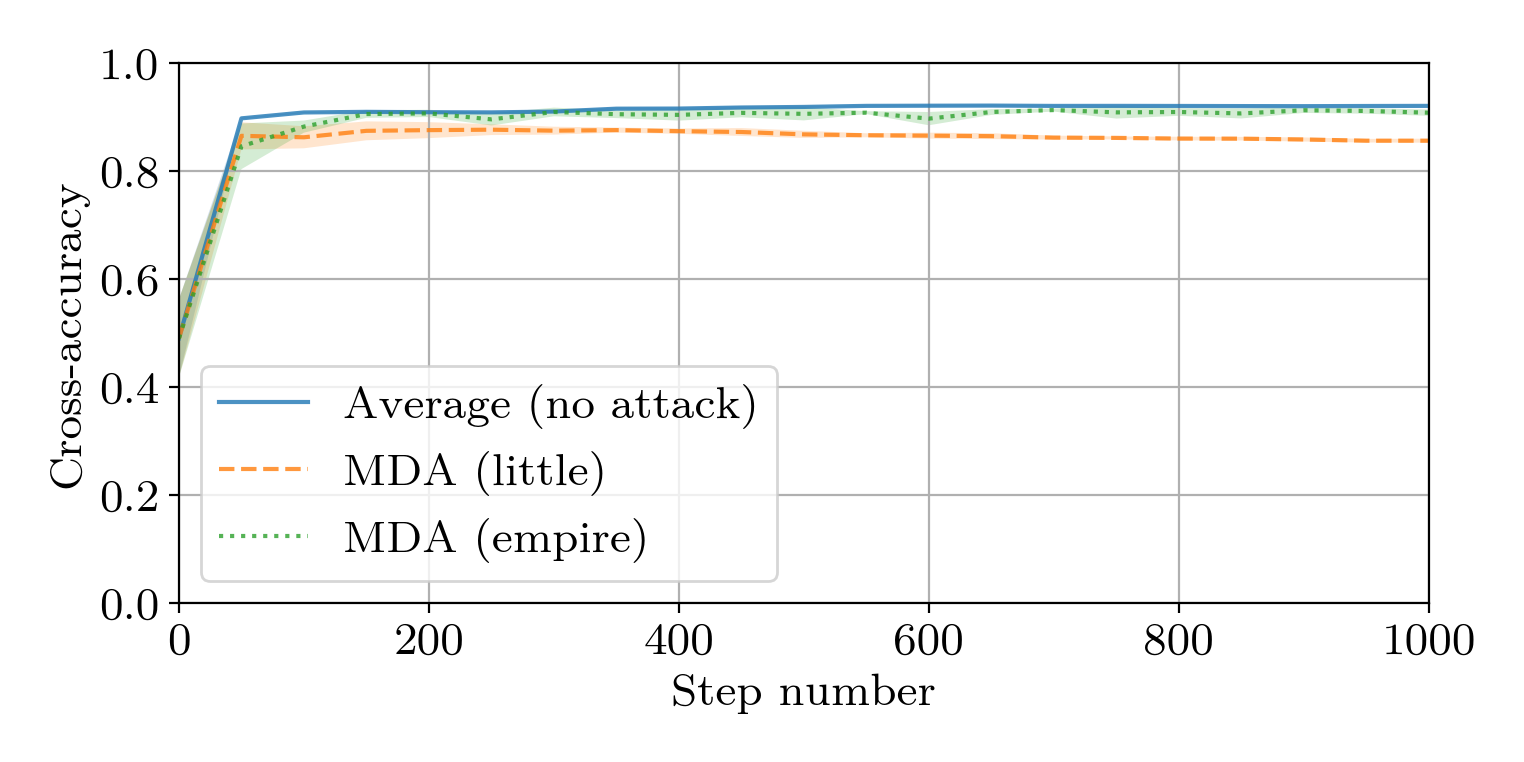}%
    \includeplot{0.5\linewidth}{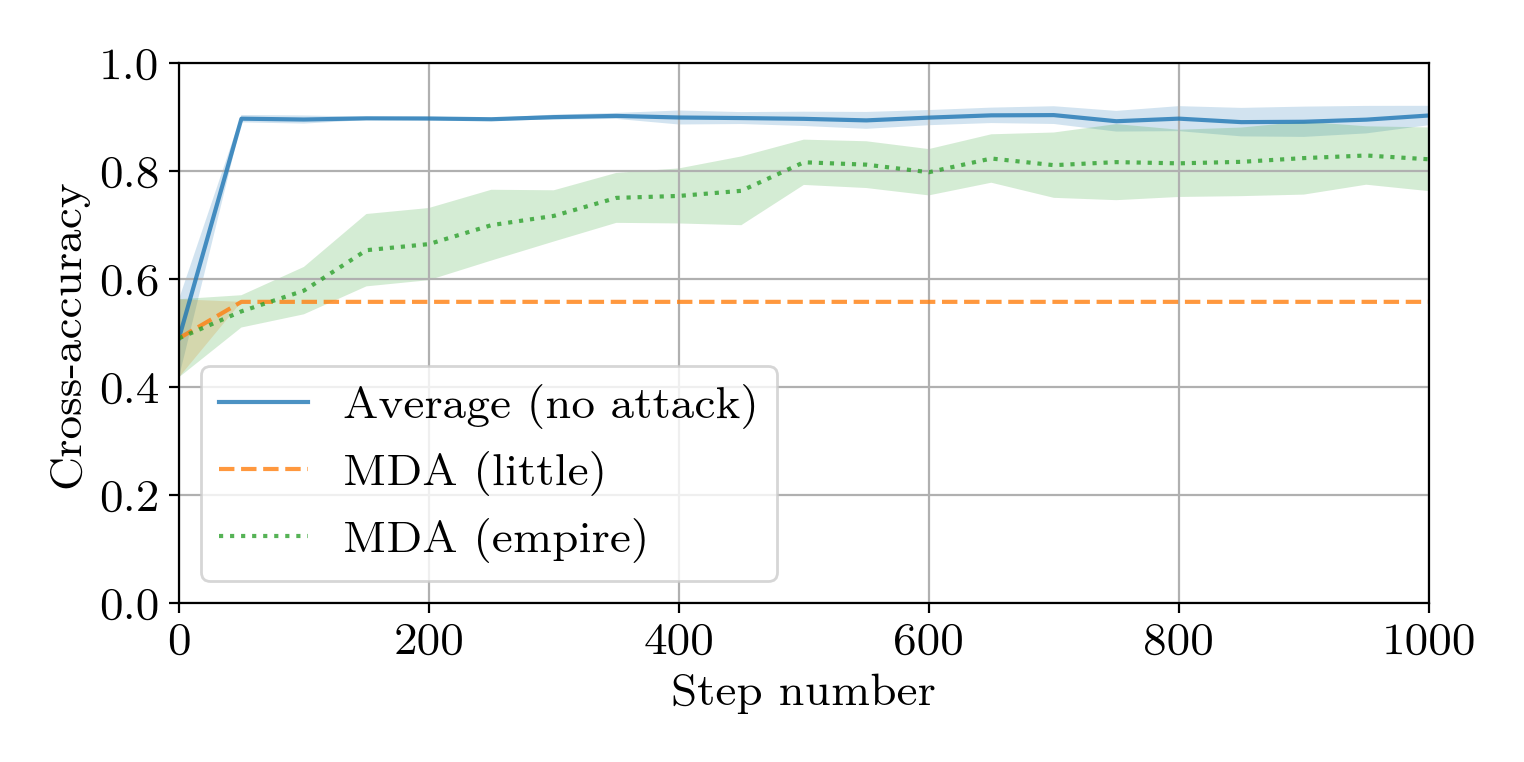}\\%
    \includeplot{0.5\linewidth}{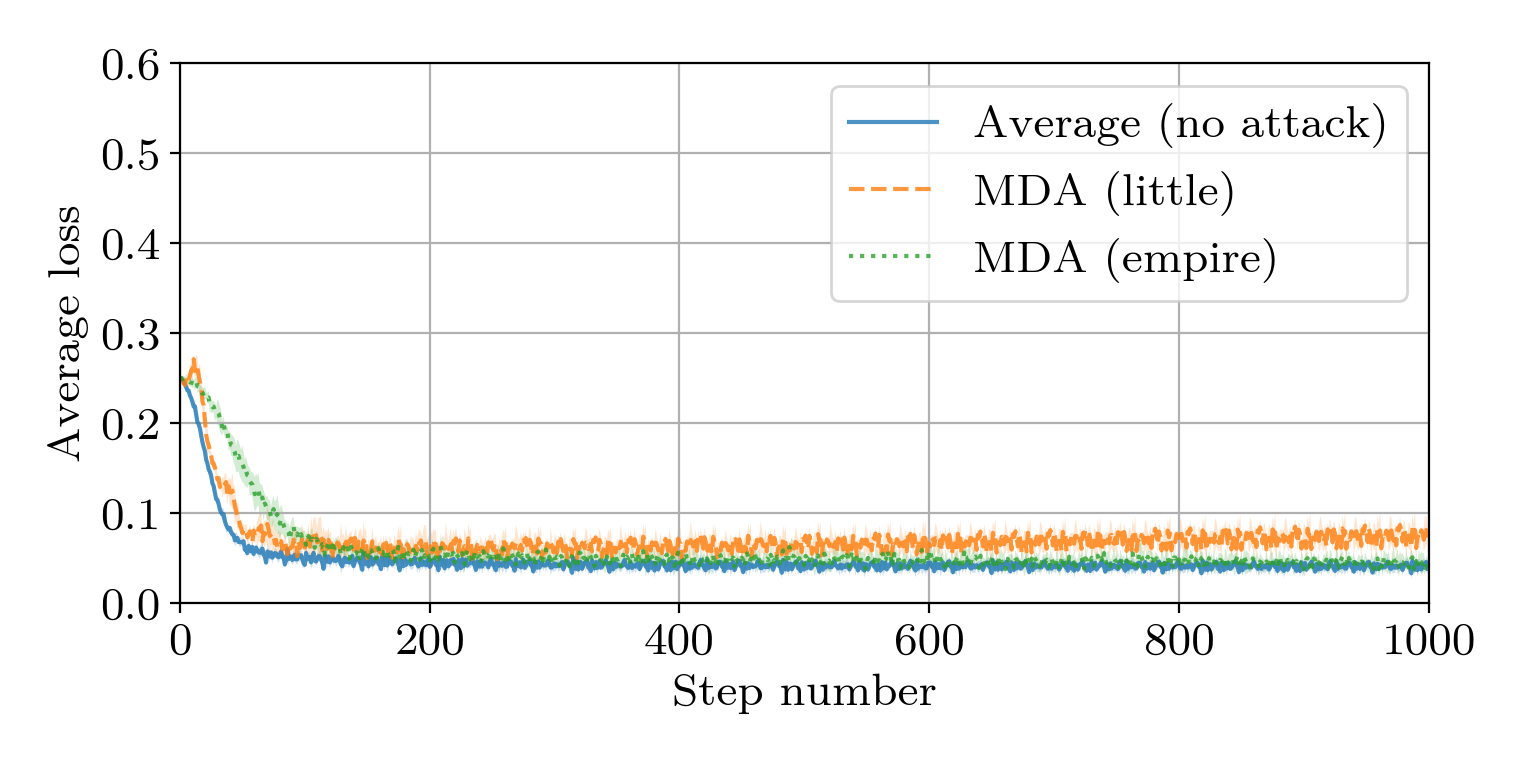}%
    \includeplot{0.5\linewidth}{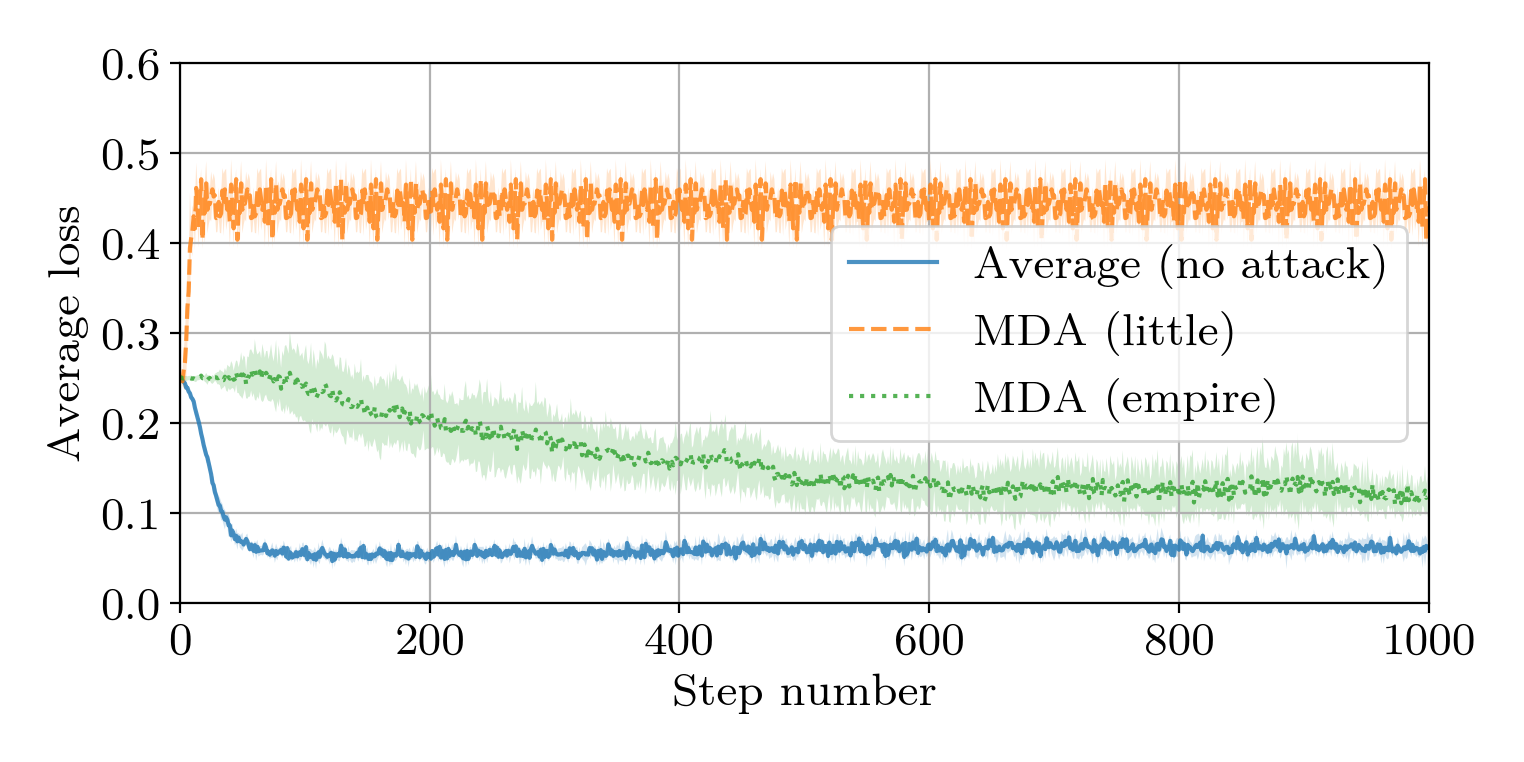}%
    \caption{In this set of experiments, $b = 50$. Without any DP, the minimum loss is reached in less than 100 steps, no matter which or whether an attack occurred. When DP noise is used, the \textit{unattacked} case remains essentially unaffected, the minimum loss being reached in 50 steps in both cases. We finally observe that, when DP noise is employed under attack, the protection provided by \brute{} (despite \brute{} offering the highest known upper bound of the VN ratio) is noticeably lowered: this is an instance of the antagonism between privacy noise and $\left(\alpha, f\right)$-Byzantine resilience.}
    \label{fig:experiments-main}
\end{figure}

\subsection{Experimental Setup}
We train a logistic regression model on the academic \textit{phishing} dataset\footnote{\url{https://www.csie.ntu.edu.tw/~cjlin/libsvmtools/datasets/}}.
Each datapoint in the \textit{phishing} dataset contains 68 features, and so our model has $d = 69$ parameters (there is an additional parameter for the bias).
We use the \textit{mean square error} as training loss.
The model is trained with SGD over $1000$ iterations, with a fixed learning rate of $\eta = 2$ and a momentum of $0.99$.
Stochastic gradients are clipped to a maximum $\ell_2$-norm of $G_{max} = 10^{-2}$.
Each worker adds a privacy noise only \textit{after} clipping the original gradient.
We set the privacy parameter $\delta = 10^{-6}$.

The phishing dataset contains 11\,055 datapoints.
We split them into \textit{training} and \textit{testing} sets, containing 8\,400 and 2\,655 datapoints respectively.
We measure (1) the cross-accuracy achieved by the logistic regression over the entire testing set every 50 steps, and at each step (2) the average loss achieved by the model over the training datapoints sampled by the honest workers.
The values of the privacy parameter $\epsilon$ and the training batch size $b$ are varied across experiments; refer to the full version of the paper\footnote{\url{https://arxiv.org/abs/2102.08166v1}}.

\begin{figure}
    \centering
    \makebox[0.5\linewidth][c]{\small{}Without privacy noise}%
    \makebox[0.5\linewidth][c]{\small{}With privacy noise ($\epsilon = 0.2$)}\\[0.1mm]%
    \includeplot{0.5\linewidth}{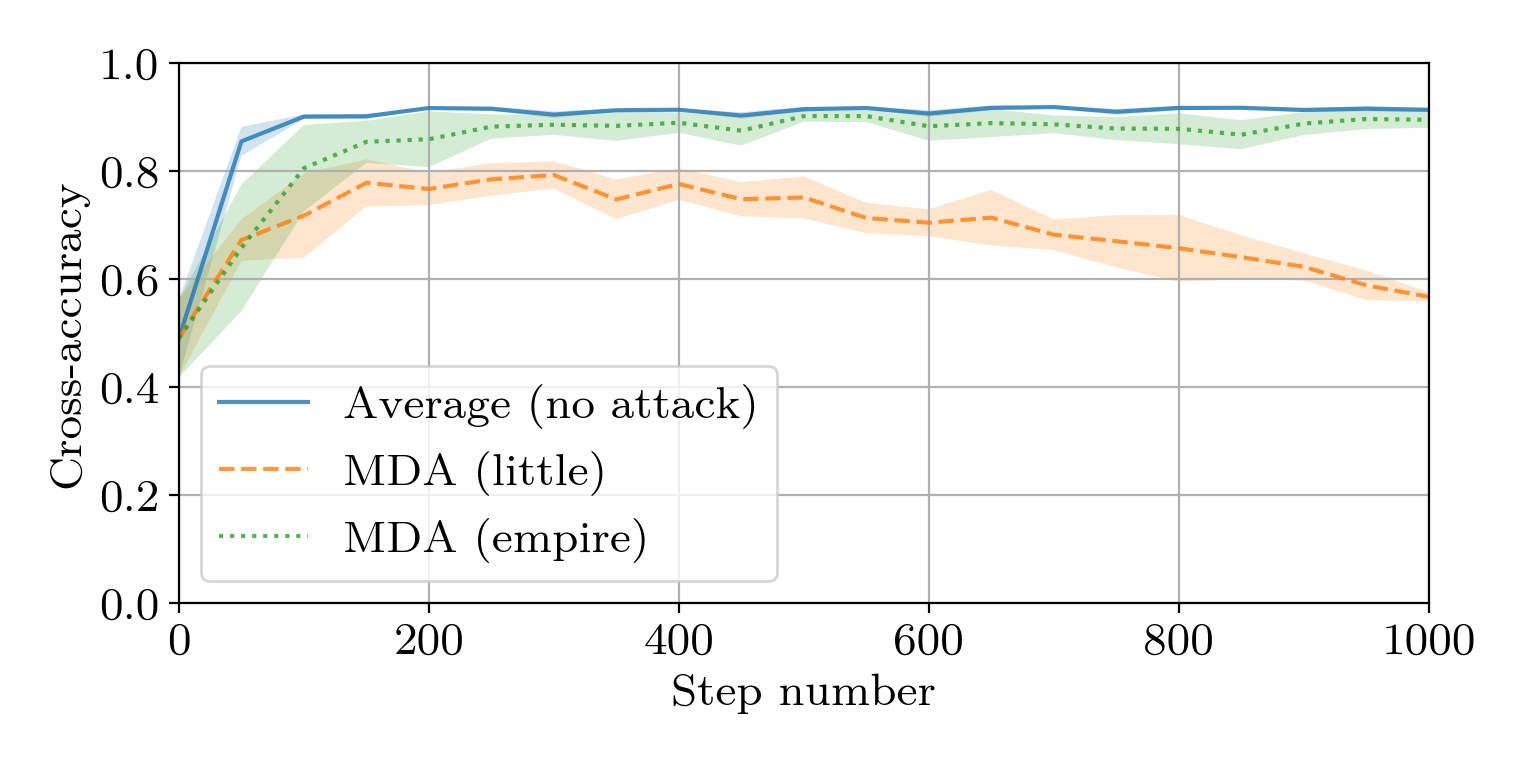}%
    \includeplot{0.5\linewidth}{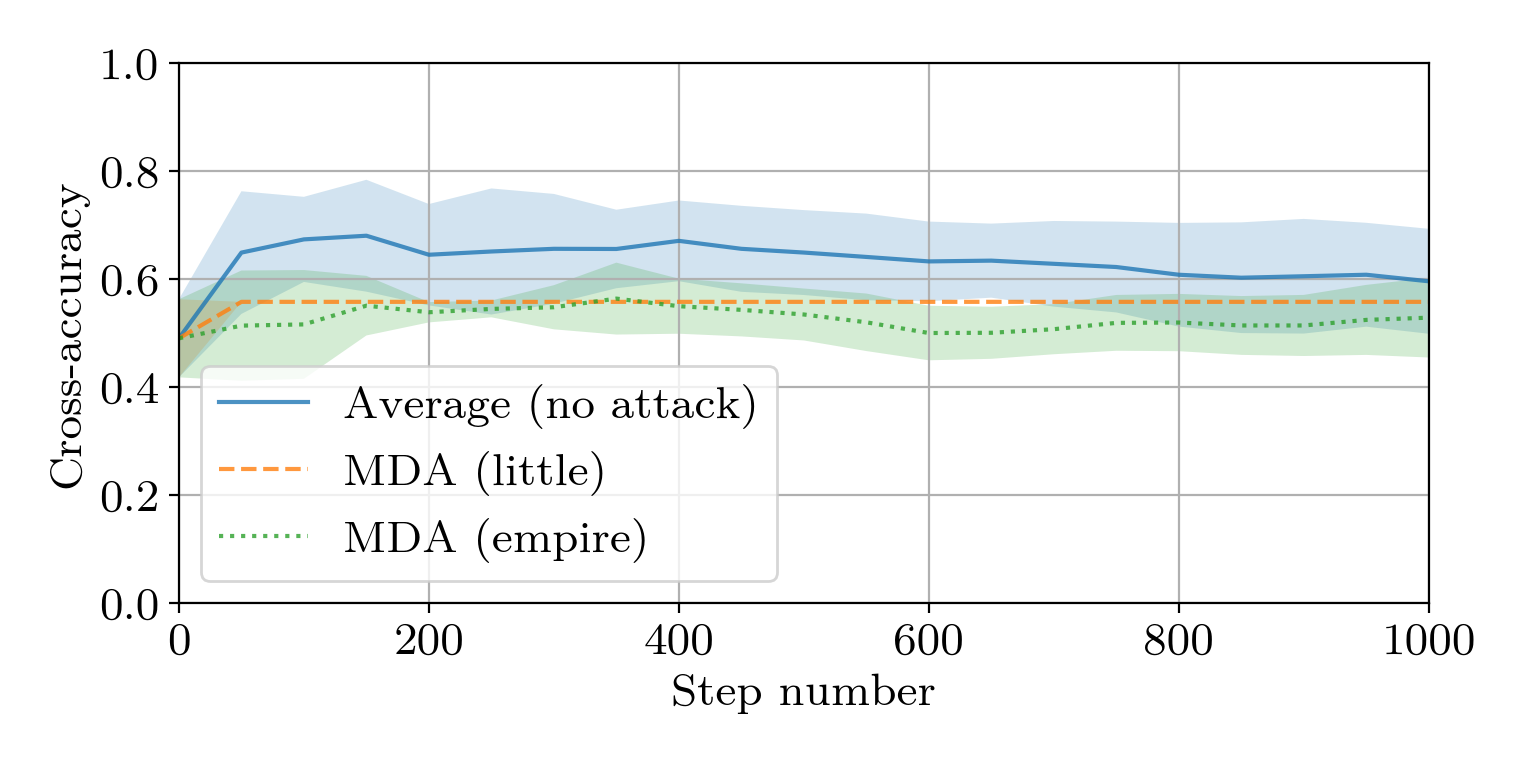}\\%
    \includeplot{0.5\linewidth}{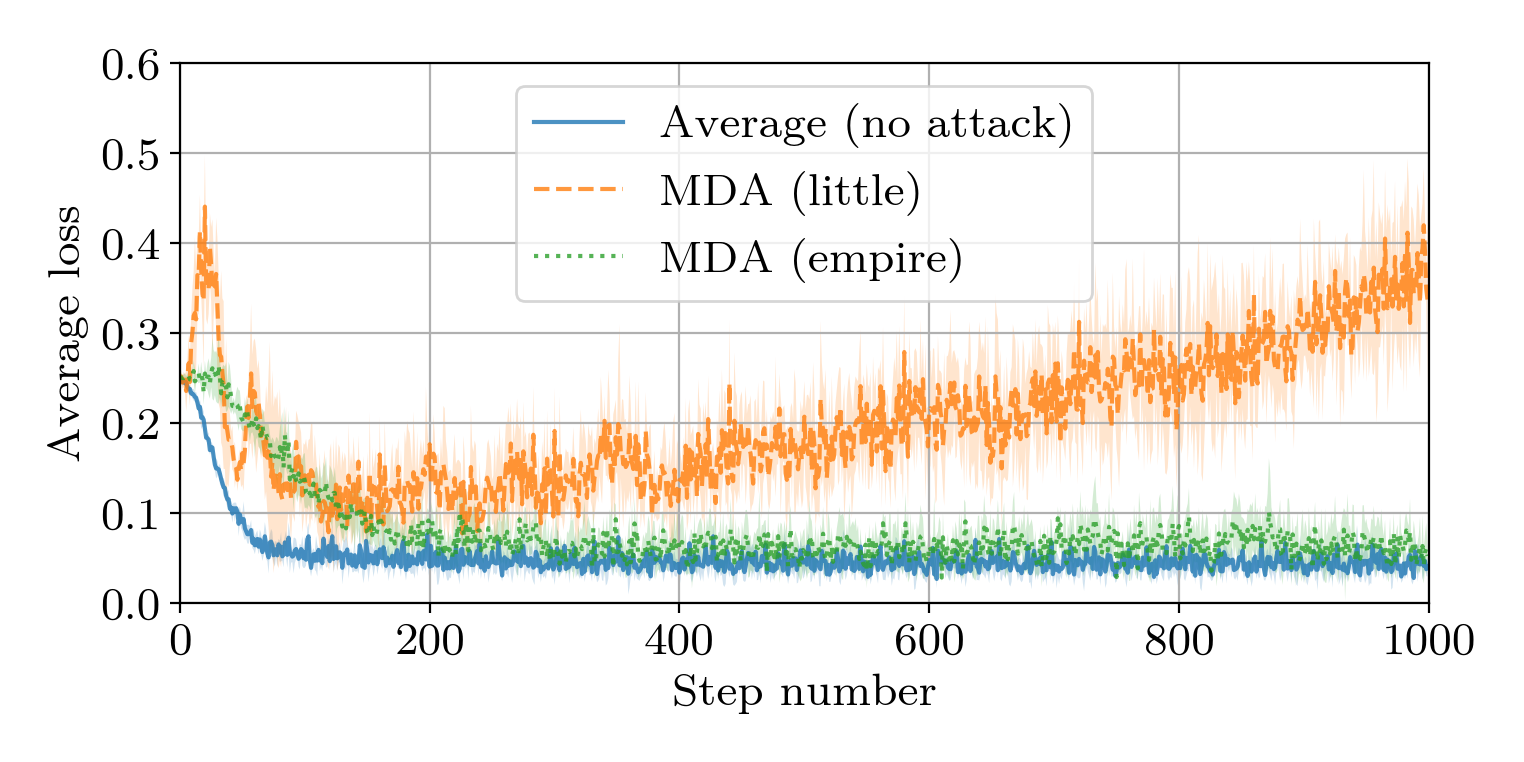}%
    \includeplot{0.5\linewidth}{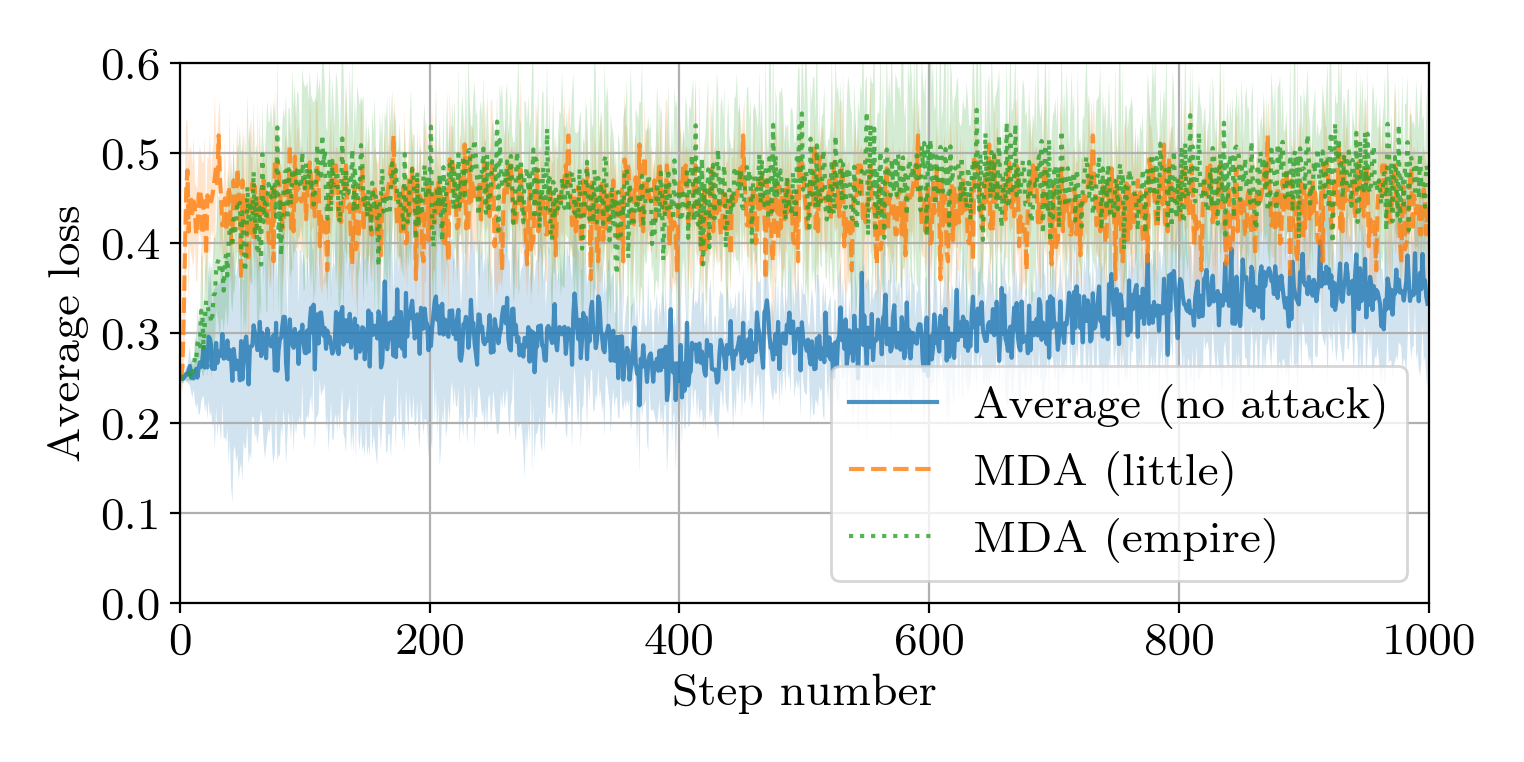}%
    \caption{Compared to Fig.\ \ref{fig:experiments-main}, we set a much smaller training batch size $b = 10$.
    Decreasing $b$ increases the variance of honest gradients.
    While the \textit{unattacked} setting without DP noise remains mostly unaffected here, adding noise significantly hampers the training even without attack.
    This ``extreme'' configuration (with Fig.\ \ref{fig:experiments-extreme-high} on the other side of the spectrum) reminds that in practice (i.e., in a finite number of steps) mere \textit{convergence}, even without attack or DP, already requires a sufficiently low variance. Both Byzantine resilience and DP step this intrinsic requirement up.}
    \label{fig:experiments-extreme-low}
\end{figure}

\begin{figure}
    \centering
    \makebox[0.5\linewidth][c]{\small{}Without privacy noise}%
    \makebox[0.5\linewidth][c]{\small{}With privacy noise ($\epsilon = 0.2$)}\\[0.1mm]%
    \includeplot{0.5\linewidth}{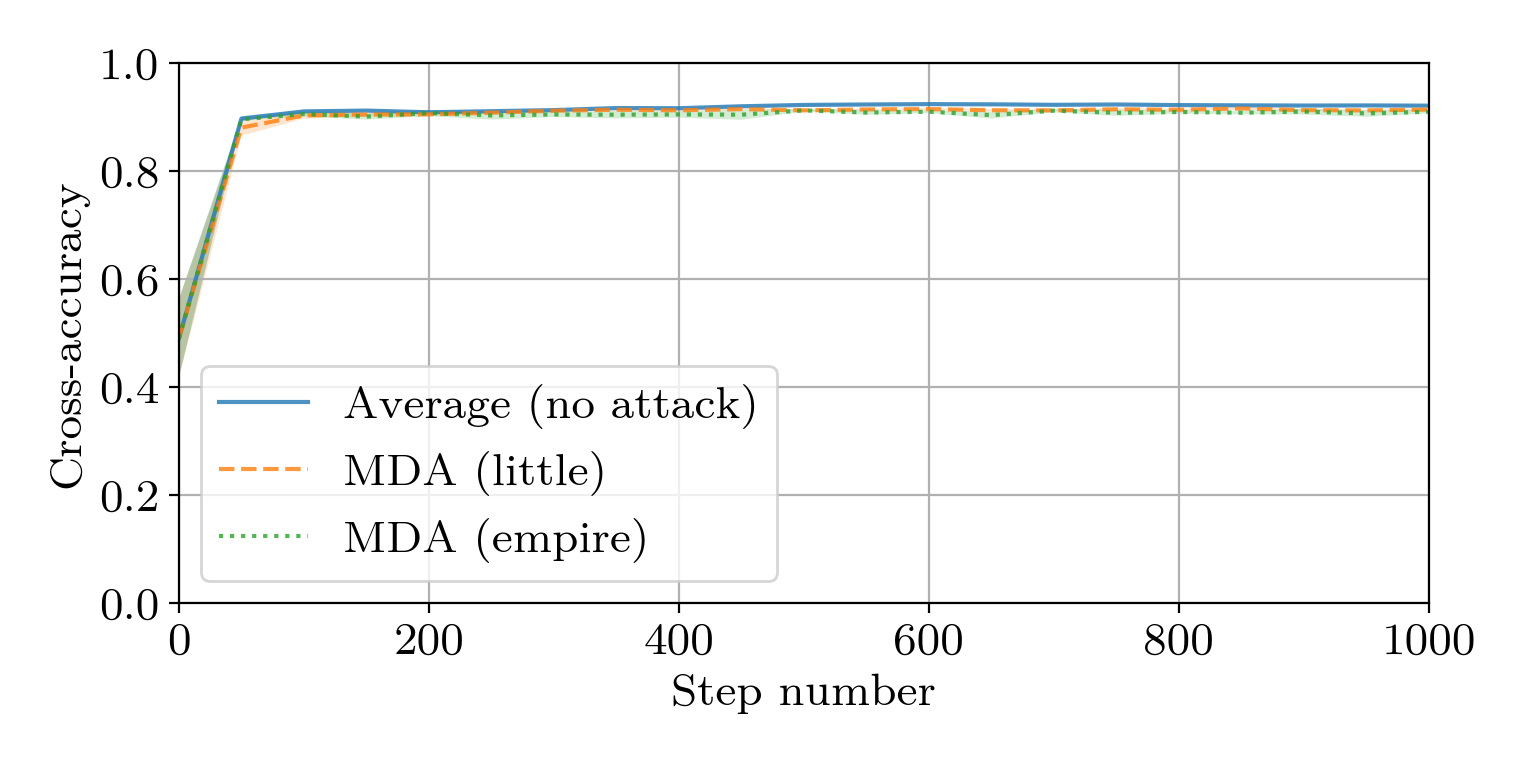}%
    \includeplot{0.5\linewidth}{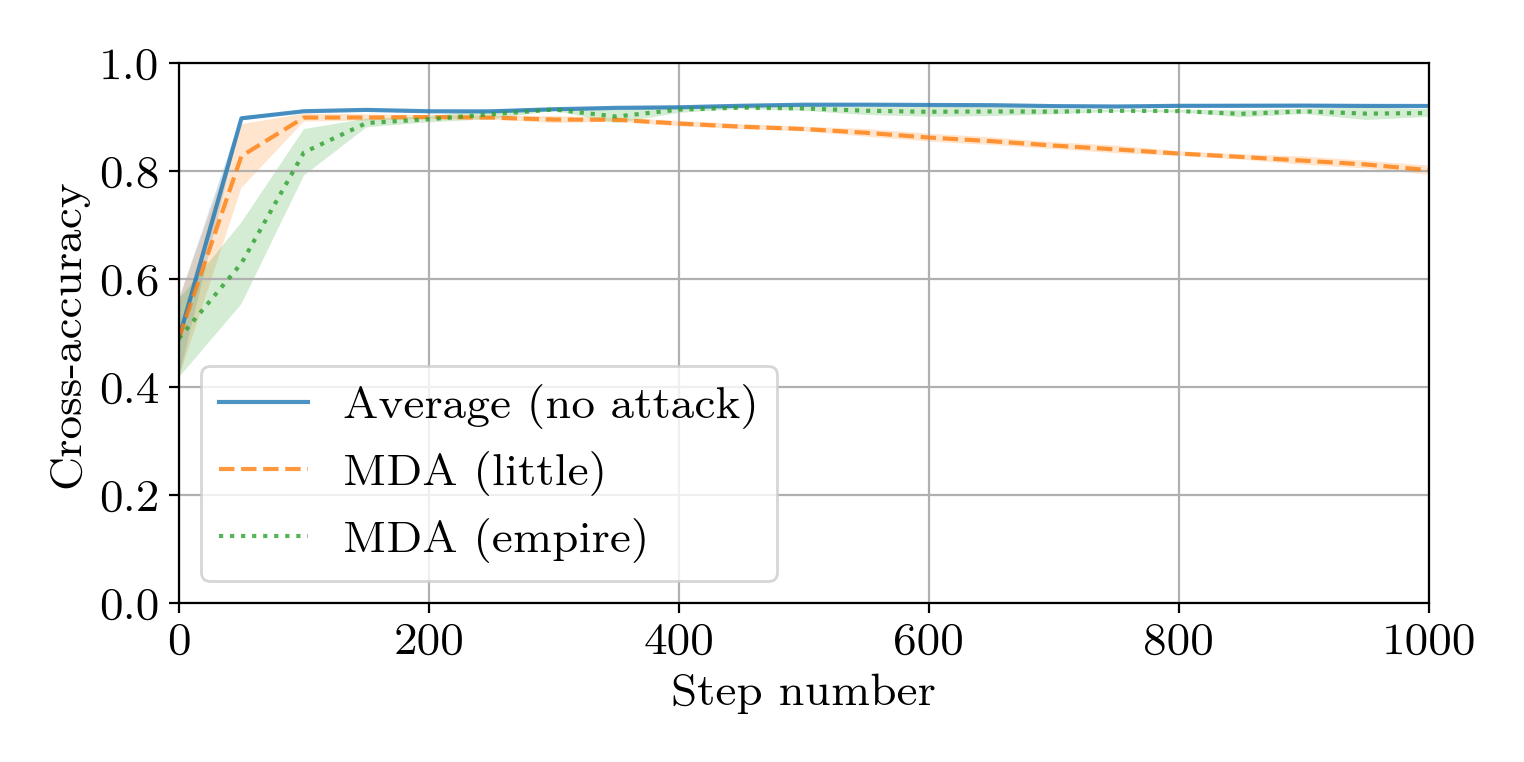}\\%
    \includeplot{0.5\linewidth}{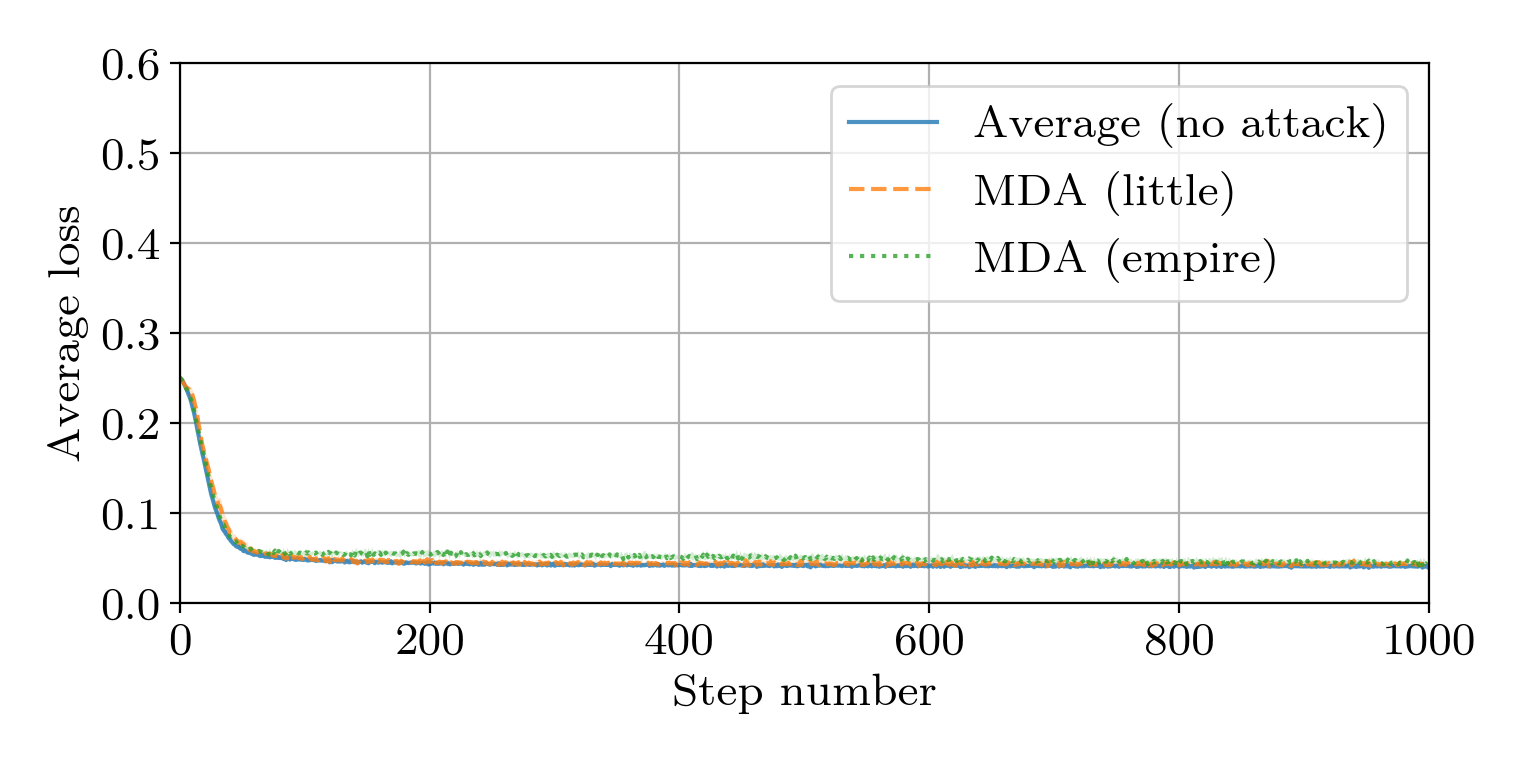}%
    \includeplot{0.5\linewidth}{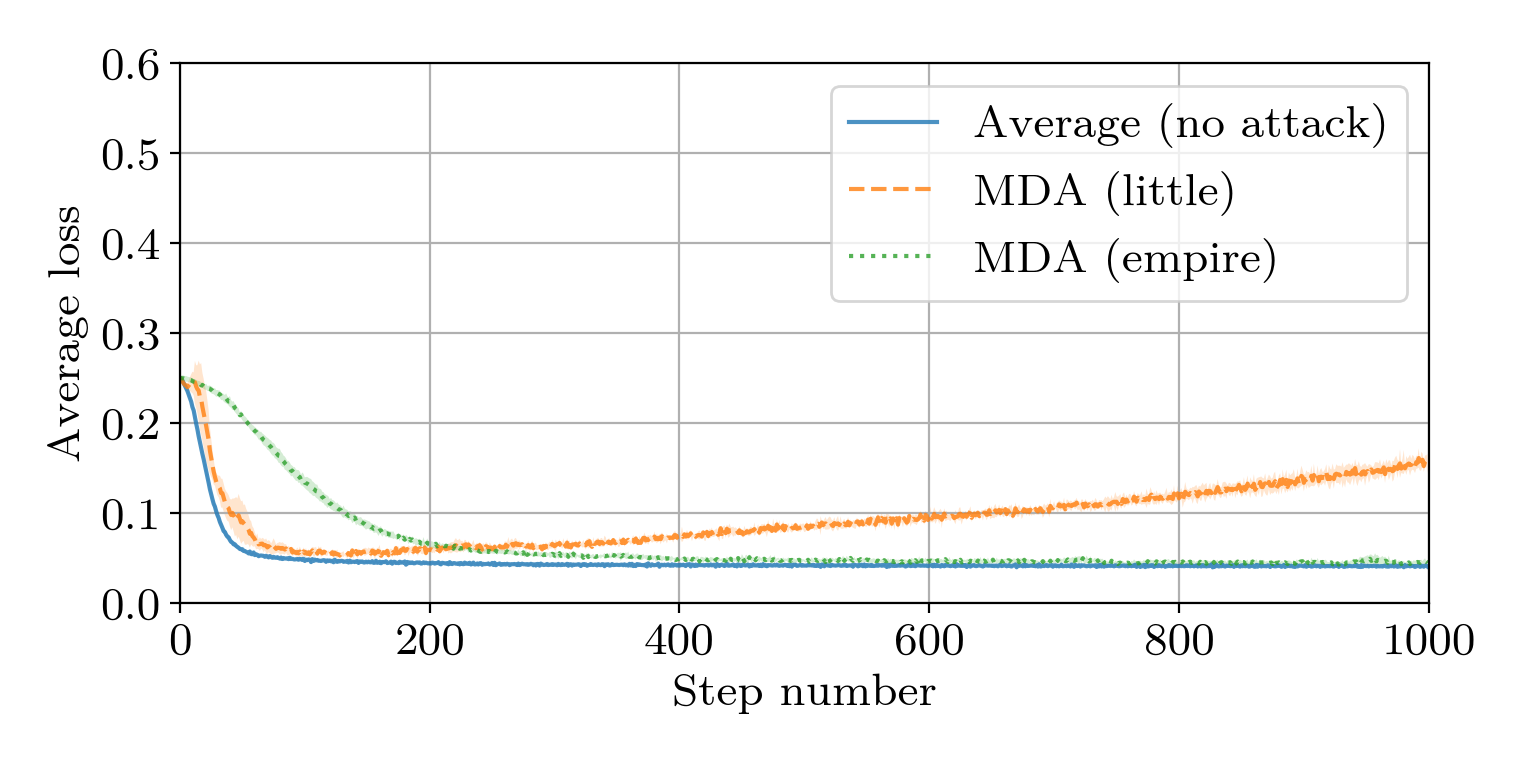}%
    \caption{Compared to Fig.\ \ref{fig:experiments-main}, we set a much larger training batch size $b = 500$.
    Increasing $b$ decreases the variance of the honest gradients.
    Training remains unaffected: the minimum loss and maximum accuracy achieved by the \textit{unattacked}, non differentially-private runs are also achieved under attack and/or with the privacy noise.
    This ``extreme'' configuration (with Fig.\ \ref{fig:experiments-extreme-low} on the other side of the spectrum) highlights that this \textit{incompatibility} between DP and Byzantine resilience is not about an \textit{impossibility} in the strict sense of the term, but more about an \textit{antagonism} that may prove difficult to resolve in practice ($b = 500$ is unreasonably high considering the minimal loss can be attained with only $b = 10$, i.e., $50\times$ less training samples).}
    \label{fig:experiments-extreme-high}
\end{figure}

\paragraphspace{}
\noindent
\paragraph{State-of-the-art attacks}
We consider two recent state-of-the-art attacks \cite{distributed-momentum}.
Both attacks follow the same core principle.
Let $\nu \in \setr{}_{\ge 0}$ be a non-negative factor, and $\attack{t} \in \setr{}^d$ an \textit{attack vector} which depends on the attack used (see below for possible values of $\attack{t}$).
At each step $t$, each Byzantine worker submits the same Byzantine gradient: $\overline{\gradient{}{t}} + \nu \, \attack{t}$,
where $\overline{\gradient{}{t}}$ is an approximation of the real gradient $\nabla{}\realloss{\params{}{t}}$ at step $t$, and $\nu$ is a constant (see below).

\begin{itemize}[nolistsep,noitemsep,leftmargin=*]
    \item{\textit{A Little is Enough} \citep{little}:
    in this attack, each Byzantine worker submits $\overline{\gradient{}{t}} + \nu \, \attack{t}$, with $\attack{t} \triangleq -\sigma_t$ the opposite of the coordinate-wise standard deviation of the honest gradient distribution.
    Our experiments use $\nu = 1.5$, as proposed by the original paper.}
    \item{\textit{Fall of Empires} \citep{empire}:
    each Byzantine worker submits $\left( 1 - \nu \right) \overline{\gradient{}{t}}$, i.e., $a_t \triangleq -\overline{\gradient{}{t}}$.
    The original paper tested $\nu' \in \{$-10, -1, 0, 0.1, 0.2, 0.5, 1, 10, 100$\}$,
    our experiments use\footnote{This factor made this attack consistently successful in the original paper.} $\nu = 1.1$, corresponding in the notation of the original paper to $\nu' \triangleq -\left( 1 - \nu \right) = 0.1$.}
\end{itemize}

\paragraphspace{}
\noindent
\paragraph{Choice of GAR}
In all our experiments, the parameter server uses \brute{} to aggregate the received gradients.
\brute{} has one of the largest VN ratio upper bounds ($k_F(n,f)$, Eq.~\eqref{VN_Ratio_Condition}) among known $(\alpha,f)$-Byzantine resilient GARs\footnote{We believe that no other $(\alpha,f)$-Byzantine resilient GAR has a higher tolerance than \brute{} in practice, but since the literature on the matter is quickly growing, we simply prefer to remain cautious.}, and the largest among the presented GARs (Section \ref{byzRes}).

\paragraphspace{}
\noindent
\paragraph{Distributed setting}
We set a fixed total of $n = 11$ workers, among which $f = 5$ can be Byzantine.
When the server uses \brute{} to aggregate gradients, the Byzantine workers implement the same attack, either \cite{little} or \cite{empire} (see above).
When \textit{averaging} is used, the $f$ workers do not implement any attack and behave as honest workers.

\subsection{Experimental Results}\label{sec:exp-results}
We test the variation of the following parameters: (1) the batch size $b$ used to estimate each gradient, (2) the privacy parameter $\epsilon$ (see the full version of paper), and (3) the attack used (either \cite{little} or \cite{empire}).
For each combination of these 3 parameters, we compare the results obtained (a) without any privacy or attack, (b) under attack without DP noise, (c) with DP noise but without attack and (d) both under attack and with DP noise.
We report on the average and standard deviation of both the cross-accuracy and the average loss achieved by the model.

\paragraphspace{}
\noindent
\paragraph{Reproducibility}
Each experimental setup is repeated $5$ times, with specified seeds (in $1$ to $5$).
We provide the code employed in this paper~\cite{anonymized-code}.
All our results, including the graphs, are reproducible in one command (see the README).

\paragraphspace{}
\noindent
\paragraph{Result overview}
While a larger, more systematic sweep of hyperparameters is available in the full version of the paper, here we report on a representative subset of behaviors.
Fig.\ \ref{fig:experiments-main} displays the evolution of the considered metrics, under attack or not, with or without DP, with an arguably reasonable training batch size ($b = 50$) for the ML task at hand\footnote{Among all tested batch sizes, $b = 50$ is the smallest for which both the attack/non-DP and non-attack/DP settings achieve performances fairly similar to the non-attacked/non-DP case; see the full version of the paper to compare the selected set with other sets of hyperparameters.}.
We also report on two opposed ``extreme'' settings, for which either no DP/Byzantine behavior alone can be tolerated (Fig.\ \ref{fig:experiments-extreme-low}), or both DP and Byzantine behaviors can be tolerated together without affecting much the accuracy/loss of the model (Fig.\ \ref{fig:experiments-extreme-high}).

The practical difficulty to reconcile DP with ($\alpha, f$)-Byzantine resilience is conspicuous in our results, even with the purposely small ($d = 69$) model at hand.
The training batch size for which both notions can be combined ($b = 500$, Fig.~\ref{fig:experiments-extreme-high}) is 10 times larger than the training batch size for which either of the two techniques can be used alone ($b = 50$, Fig.~\ref{fig:experiments-main}), and at least 50 times larger than the batch size necessary to achieve mere convergence without any DP or Byzantine resilience ($b = 10$, Fig.~\ref{fig:experiments-extreme-low}).

The other takeaway from our experimental results, particularly visible with the hyperparameter sweep featured in the full version, is that slightly larger privacy noises \textit{gracefully} translates into slightly lower performances (lower accuracy and higher loss); not any abrupt decrease in performances, which could have been anticipated due to the existence of Byzantine workers.
This observation is predicted by the theory presented in Section \ref{sub:str_cvx}, as the loss of the considered ML task is convex.
So at least for convex problems, the practitioner can always trade some accuracy for some privacy (despite the potential presence of adversarial workers), and more computation time (larger training batch size) for more accuracy.

%% file: app_priv-res-tf.tex
\section{Detailed Proof of the Upper Bound in Theorem~\ref{thm:priv_res_tf}}
\label{app:priv_res_tf}

We let $F_t \triangleq F\left( o^{(1)}_t, \ldots, \, o^{(n)}_t \right)$. Recall, from Section~\ref{Gaussian}, that in each iteration $t$ if worker $W_i$ is honest then
\begin{align*}
    \noisygradient{i}{t} = h(\batch{i}{t}) + \boldsymbol{\noise{i}{t}}  ~ \text{ with } ~ \boldsymbol{\noise{i}{t}} \sim \mathcal{N}\left(0, I_d \times s^2\right) 
\end{align*}
where $h(\batch{i}{t}) \sim G_t$ such that $\E_{\xi_t} \left[G_t\right] = \nabla Q(w_t)$ and
\begin{align*}
    \E_{\xi_t} \left[ \norm{G_t - \E_t G_t}^2 \right]= \E_{\xi_t} \left[ \left(\frac{1}{b} \sum\limits_{j = 1}^{b}{\compgrad{\weight{t}}{\datapoint{i,j}{t}}} - \nabla Q (\weight{t}) \right)^2 \right]. 
\end{align*}
Therefore, $o^{(i)}_t \sim P_t$ for an honest worker $W_i$ such that $\E_{\xi_t} \left[P_t \right] = \nabla Q(w_t)$, and
\begin{multline}
    ~ \E_{\xi_t} \left[\norm{P_t - \E P_t}^2 \right] = \\
    \E_{\xi_t} \left[\left(\frac{1}{b} \sum\limits_{j = 1}^{b}{\compgrad{\weight{t}}{\datapoint{i,j}{t}}} + \boldsymbol{\noise{i}{t}} - \nabla Q (\weight{t}) \right)^2 \right].  \label{eqn:exp_o_t}
\end{multline}
As the square function $(\cdot)^2$ is convex, from~\eqref{eqn:exp_o_t} and Assumption~\ref{asp:bnd_var}, we obtain that
\begin{align}
    \E_{\xi_t} \left[ \norm{P_t - \E_{\xi_t} P_t}^2 \right] \leq \frac{\sigma^2}{b} + d s^2 . \label{eqn:o_var}
\end{align}
Recall that $\E_{\xi_t} \left[ P_t \right] = \nabla Q(w_t)$, and by Assumption~\ref{asp:bnd_norm}, $\norm{\nabla Q(w_t)} \leq G_{max}$. Thus, from~\eqref{eqn:o_var}, we obtain that  $\E_{\xi_t} \left[ \norm{P_t}^2 \right]\leq \frac{\sigma^2}{b} + d s^2 + G_{\max}^2$, and  
\begin{align}
\E_{\xi_t} \left[ \norm{P_t}\right] \leq \sqrt{ \frac{\sigma^2}{b} + d s^2 + G_{max}^2}. \label{eqn:2_order}
\end{align}

Now, recall from the definition of $(\alpha, \, f)$-Byzantine resilience in Section~\ref{byzRes}, for each $ t \in \{ 1, \ldots, \, T \}$,
\begin{align}
    \langle \E_{\xi_t} \left[F_t \right], \, \nabla Q(w_t) \rangle & \geq (1-\sin\alpha) \norm{\nabla Q(w_t)}^2 > 0. \label{eqn:pi_byzres}
\end{align}
Moreover, there exists a positive real-value $c$ such that for each $ t \in \{ 1, \ldots, \, T \}$,
\begin{align}
    \E_{\xi_t} \left[ \norm{F_t}^2 \right] \leq c \max \left\{ \E_{\xi_t} \left[ \norm{P_t}^2 \right], \, \E_{\xi_t} \left[ \norm{P_t} \right] \right\}. \label{eqn:byrres_mom}
\end{align}
It is safe to assume that $\frac{\sigma^2}{b} + d s^2 + G_{max}^2 \geq 1$. Thus, substituting from~\eqref{eqn:2_order} in~\eqref{eqn:byrres_mom}, we obtain that, for each $ t \in \{ 1, \ldots, \, T \}$,
\begin{align}
    \E_{\xi_t} \left[ \norm{F_t}^2 \right] \leq c \left(\frac{\sigma^2}{b} + d s^2 + G_{\max}^2 \right). \label{eqn:byrres_mom_2}
\end{align}
~

From~\eqref{eqn:update_thm}, and Assumption~\ref{asp:lip}, we obtain that, for all $t$,
\begin{align*}
    Q(w_{t+1}) \leq Q(w_t)  - \gamma_t \iprod{\nabla Q(w_t)}{ F_t } + \frac{1}{2} \mu \gamma^2_t \norm{F_t}^2 . 
\end{align*}
Note that $\E_{\xi_t}\left[ Q(w_t)\right] = Q(w_t)$. Upon taking expectation $\E_{\xi_t}$ on both sides above, we obtain that, for all $t$,
\begin{multline}
    \E_{\xi_t} \left[Q(w_{t+1}) \right] \leq Q(w_t)  - \gamma_t \iprod{\nabla Q(w_t)}{ \E_{\xi_t} \left[ F_t \right]}\\
    + \frac{1}{2} \mu \gamma^2_t \E_{\xi_t} \left[ \norm{F_t}^2 \right] . \label{eqn:upd_1}
\end{multline}
Substituting from~\eqref{eqn:pi_byzres} in~\eqref{eqn:upd_1} implies that, for $t = 1, \ldots, \, T$,
\begin{multline}
    \E_{\xi_t} \left[ Q(w_{t+1}) \right] \leq Q(w_t)  - \gamma_t (1 - \sin \alpha) \norm{\nabla Q(w_t)}^2\\
    + \frac{1}{2} \mu \gamma^2_t \E_{\xi_t} \norm{F_t}^2. \label{eqn:upd_2}
\end{multline}
From Assumption~\ref{asp:str_cvx}, $\norm{\nabla Q(w_t)}^2 \geq 2 \lambda \, \left(Q(w_t) - Q^* \right)$. Substituting this above implies that
\begin{multline}
\E_{\xi_t} \left[Q(w_{t+1})\right] - Q^* \leq \left(1 - 2 \lambda (1 - \sin \alpha) \gamma_t \right) \left( Q(w_{t}) - Q^* \right)\\
+ \frac{1}{2} \mu \gamma^2_t \E_{\xi_t} \left[\norm{F_t}^2 \right] , \quad \text{ for } t = 1, \ldots, \, T. \label{eqn:upd_3}
\end{multline}
Suppose that $t > 1$. Taking expectation $\E_{\xi_1} \dots \E_{\xi_{t-1}}$ on both sides in~\eqref{eqn:upd_3}, we obtain that
\begin{multline}
\E_{t} \left[Q(w_{t+1})\right] - Q^* \leq \left(1 - 2 \lambda (1 - \sin \alpha) \gamma_t \right) \left( \E_{t-1}  \left[Q(w_{t}) \right] - Q^* \right)\\
+ \frac{1}{2} \mu \gamma^2_t \E_{t} \left[\norm{F_t}^2 \right] , \quad \text{ for } t = 2, \ldots, \, T. \label{eqn:upd_3_3}
\end{multline}
Substituting from~\eqref{eqn:byrres_mom_2} in~\eqref{eqn:upd_3_3}, we obtain that
\begin{multline*}
\E_t \left[Q(w_{t+1}) \right] - Q^* \leq \left(1 - 2 \lambda (1 - \sin \alpha) \gamma_t \right) \left( \E_{t-1} \left[Q(w_{t}) \right] - Q^* \right)\\
+ \frac{1}{2} \mu c \left(\frac{\sigma^2}{b} + d s^2 + G_{\max}^2 \right) \gamma^2_t, \quad \text{ for } t = 2, \ldots, \, T. 
\end{multline*}
Substituting $\gamma_t = \frac{1}{\lambda (1 - \sin \alpha) \, t}$ above we obtain that
\begin{multline*}
\E_t \left[Q(w_{t+1}) \right] - Q^* \leq \left(1 - \frac{2}{t} \right) \left( \E_{t-1} \left[Q(w_{t}) \right] - Q^* \right)\\
+ \frac{1}{t^2} \left(\frac{\mu \, c}{2 \lambda^2 (1 - \sin \alpha)^2} \right) \, \left(\frac{\sigma^2}{b} + d s^2 + G_{\max}^2 \right), \quad \text{ for } t = 2, \ldots, \, T.
\end{multline*}
Finally, using induction it is easy to show that
\begin{multline}
    \E_T \left[ Q(w_{T+1}) \right] - Q^*  \leq \\
    \frac{1}{T + 1} \left(\frac{\mu \, c}{2 \lambda^2 (1 - \sin \alpha)^2} \right) \, \left(\frac{\sigma^2}{b} + d s^2 + G_{\max}^2 \right) \label{eqn:no-dp-rate}
\end{multline}
Therefore, $\E_t \left[ Q(w_{T+1}) \right] - Q^* \in O\left( \frac{d \, s^2}{ T} \right)$.
Recall, from Section~\ref{Gaussian}, that $s= \frac{2 G_{max}\sqrt{2 \log(1.25/\delta)}}{b \epsilon}$. Hence, 
\[\E_T \left[ Q(w_{T+1}) \right] - Q^* \in O\left( \frac{d \log(1/\delta)}{ T b^2 \epsilon^2} \right) .\]